%% file: neurips_2026.tex
\documentclass{article}

\usepackage[preprint]{neurips_2026}
\usepackage{algorithm}
\usepackage{algpseudocode}
\input{math_commands.tex}

\usepackage[utf8]{inputenc} 
\usepackage[T1]{fontenc}    
\usepackage{hyperref}       
\usepackage{url}            
\usepackage{booktabs}       
\usepackage{amsfonts}       
\usepackage{nicefrac}       
\usepackage{microtype}      
\usepackage{xcolor}         

\usepackage{hyperref}
\usepackage{url}
\usepackage{microtype}
\usepackage{graphicx}
\usepackage{subcaption}
\usepackage{booktabs} 
\usepackage{amsmath}
\usepackage{amsthm}
\newtheorem{definition}{Definition}

\newtheorem{proposition}{Proposition}
\newtheorem{example}{Example}

\theoremstyle{remark}
\newtheorem{remark}{Remark}
\usepackage{comment}
\usepackage{amsmath}
\usepackage{amssymb}
\usepackage{mathtools}
\usepackage{amsthm}

\title{When Agents Say One Thing and Do Another: Validating Elicited Beliefs from LLMs}

  \author{
Khurram Yamin$^{1}$\thanks{Correspondence to: \texttt{kyamin@andrew.cmu.edu}}
\quad
Jingjing Tang$^{1}$
\quad
Santiago Cortes-Gomez$^{1}$ \\
\textbf{Amit Sharma$^{2}$
\quad
Eric Horvitz$^{2}$
\quad
Bryan Wilder$^{1}$} \\
$^{1}$Carnegie Mellon University
\qquad
$^{2}$Microsoft Research
}

\begin{document}

\maketitle

\begin{abstract}
Large language models (LLMs) are increasingly deployed in high-stakes settings where good decisions require forming beliefs over the probability of unknown outcomes. However, it is unclear whether LLMs act as if they hold coherent beliefs when making decisions, or if so, how we could validate models' reports of such beliefs. We propose a decision-theoretic framework that elicits both probability judgments and decisions from an agent and tests their mutual consistency. Formally, our methods characterize whether it is possible for the actions to be produced by a ``near-rational" decision maker who holds the elicited probability as their true belief. We show that, perhaps surprisingly, this formalization implies empirically testable conditions even without any assumption about the agent's utility function. Applying our framework to stylized clinical diagnosis tasks, we find that models' reported beliefs are demonstrably imperfect summaries of the information revealed in their decisions, but that the discrepancies are small for the strongest models. 
\end{abstract}

\section{Introduction}

LLMs are increasingly asked to make decisions under uncertainty, where success requires forming probabilistic beliefs about unknowns in the world. For example, agents for medical tasks must reason about a patient's latent health state in order to trade off the costs and benefits of different actions. How can we tell whether an LLM acts as if it holds coherent beliefs when making decisions? We could prompt the model to output a belief -- e.g.\ the probability that the patient has a particular diagnosis -- but it is entirely possible that the model's response is simply an individual text output disconnected from any other output or decision the model might make.  Even defining what it could mean for an LLM to hold a belief is highly nontrivial. Nevertheless, validated methods for belief elicitation would be highly useful. For example, elicited beliefs would allow LLMs to provide a legible explanation of why they make particular decisions, an important component of safety and transparency in high-stakes settings. They would also enable diagnosis of failures. For example, a model developer could tell whether a medical agent fails to ask a patient followup questions \cite{liu2024large,johri2025evaluation} because it holds overconfident beliefs about the diagnosis, or because it perceives a cost to burdening the user with questions, and use this understanding to design appropriate remedies.             

Recent work has largely focused on the accuracy and calibration of expressed probabilities \citep{ulmer2024calibrating, wang2024calibrating,cruz2024evaluating}. These questions are important, but they are conceptually distinct from the one we study here. On the one hand, a decision maker may hold inaccurate or miscalibrated beliefs in good faith. On the other hand, an LLM may express probabilities that fail to reflect the quantities actually governing its decisions. For example, finetuning or preference optimization may induce the model to verbalize systematically distorted probabilities.

We develop formally-grounded strategies for validating whether probabilistic beliefs elicited from LLMs are a useful explanation for their decisions. Our strategy centers on jointly eliciting beliefs and decisions in tasks where optimal actions depend on the relevant probabilities. We propose two characteristics that elicited beliefs should satisfy in order to be a meaningful description of how an agent makes decisions. First, they should be a \textit{sufficient statistic} which encapsulates all decision-relevant information the agent has about the unknowns. Second, they should be \textit{action-guiding} in the sense that the agent's probability of selecting different actions shifts systematically in response to the elicited belief. We provide a mathematical formalization of these conditions and show how each can be empirically tested in a black-box manner using only data on elicited beliefs and decisions from an agent. We then show a specific sense in which these conditions fully characterize what is possible for black-box testing of elicited beliefs: they are provably necessary and sufficient for there to exist an agent within an expansive class of ``near-rational" decision-makers who holds the elicited belief as a true subjective probability and is observationally equivalent to the LLM.

Finally, we instantiate this framework empirically across multiple LLM families on clinically grounded diagnostic tasks spanning four medical settings. We find that elicited beliefs are provably imperfect, in that models have additional decision-relevant information about a patient's diagnosis which is revealed in their actions but not communicated in the elicited belief. However, the degree of discrepancy is relatively small for the strongest models, indicating that they are -- at least observationally -- close to an agent whose actions are coherently attributable to their elicited beliefs.


\section{Related Work}

\textbf{Belief elicitation from humans:} A long literature in psychology and economics seeks to elicit subjective probabilities, often using monetary incentives to encourage truthful reporting. Reviews \cite{schlag2015penny,charness2021experimental} emphasize that validation is difficult and arguably ill-posed: beliefs are unobservable and may be shaped by elicitation itself, so methods are often judged by how well elicited beliefs predict actions. Our approach is similar in spirit, but LLMs pose different constraints. Human studies typically assume preferences, such as more money being better, to design proper scoring rules \cite{savage1971elicitation,gneiting2007strictly} or test best responses in games \cite{nyarko2002experimental,rey2009equilibrium}. For LLMs, assuming a specific utility is strong and hard to justify, and credibly incentivizing a model in simulation is difficult. We therefore remain agnostic about the model’s utility. Finally, whereas belief elicitation for humans can ``contaminate'' subsequent actions \cite{croson2000thinking,blanco2010belief}, we avoid this for LLMs by eliciting beliefs and actions in separate context windows.

\textbf{Belief elicitation from LLMs:} With the advent of large language models there has been renewed interest in characterizing LLM beliefs. \citet{herrmann2024standards} provide a philosophical account, and our work can be seen as operationalizing their ``use" criterion. Other work has studied how to elicit priors from LLMs \cite{zhu2024eliciting} and whether LLMs express beliefs that are consistent with probabilistic axioms \cite{zhu2025incoherentprobabilityjudgmentslarge,freedman2025exploring}. Our work introduces a new set of tools to test coherence of beliefs in comparison to decisions. Perhaps the most related is work by \citet{pal2025incoherentbeliefsinconsistent}, who test whether LLMs make bets on events which go in the same direction as their beliefs. However, they do not introduce a formal framework for testing when beliefs and actions should correspond in a specific way. Further afield, an emerging literature tests whether LLMs' outputs are consistent with various cognitive biases seen in humans \cite{echterhoff2024cognitive,binz2023using,cheung2025large} but this work does not focus on validity of elicited beliefs.



\textbf{Mechanistic interpretability:}
A central motivation for our study is the concern that LLMs may know more than
what they reveal in verbalized probabilities. Mechanistic interpretability uses
access to models' internal representations to test such questions
\cite{sharkey2025open,bereska2404mechanistic}; for example, probes can predict
distinctions such as true vs.\ false claims from activations
\cite{azaria2023internal,marks2023geometry}. We view this work as complementary:
our tests are fully black-box and output-only, making them useful for closed
models, while many mechanistic approaches (e.g., linear probes) require
supervised labels, injecting external information and complicating a direct
interpretation of their outputs as beliefs of the model. Given the difficulty of
formally defining ``belief,'' both toolkits are valuable.

\textbf{Medical applications:}
A growing body of work evaluates large language models in clinically consequential settings, emphasizing that downstream quality depends on decision policies rather than static accuracy alone \citep{singhal2024expert,Gaber2025,Hager2024clinicalevaluation,Williams2024edllm}.   Clinical evaluations further stress that models may lack reliable metacognitive awareness of their own limitations, making naive self-reported confidence insufficient for safe triage \citep{Griot2025metacognition,Hager2024clinicalevaluation}. Evaluations have also stressed difficulties in seeking information when necessary to make a diagnosis, perhaps reflecting incorrect probabilistic reasoning \cite{johri2025evaluation,li2024mediq}. Our aim is to provide a principled framework within which to understand how models' understanding of uncertainty relates to their decisions.

\section{Methodology}
\label{sec:methodology}

\subsection{Decision-Theoretic Framework}
\label{subsec:framework}

We formalize the decision problem as follows. An environment generates an unknown state of the world
$\theta \sim P^\star(\theta)$ and observations $x \sim P^\star(x\mid \theta)$. The induced \emph{ground-truth}
posterior under the environment is $P^\star(\theta\mid x)$. After observing $x$, an agent forms its
own (potentially misspecified) \emph{subjective} posterior belief $P_S(\theta\mid x)$, which need not equal
$P^\star(\theta\mid x)$. In order for a subjective probability to have empirically testable content, it must be related to the agent's actions in some way. In the classic framework of expected utility maximization, the agent selects an action $a\in\mathcal{A}$ to maximize expected utility: $a(x)
\;=\;
\argmax_{a\in\mathcal{A}}\; \sum_{\theta \in \Theta} P_S(\theta|x) u(a,\theta).$
This optimization requires two components: (1) \textit{beliefs} about the unknown state---captured here by the
subjective posterior $P_S(\theta\mid x)$---and (2) \textit{preferences} over outcomes---captured here by the
utility $u$. In this paper, our goal is to formulate empirically testable conditions for an agent's \textit{reported} belief to represent a subjective probability in the sense that it represents a real account of why the agent acts the way that it does. Expected utility maximization is one framework linking beliefs to actions, but may represent too stringent a formal model on its own -- an agent could be usefully described as having beliefs even if it fails to exactly maximize a specific utility function. We instead propose conditions which can be justified in a broader sense. While we argue that these conditions are independently meaningful apart from any specific mathematical model, they can be formally grounded in an expansive class that captures ``near-rational" decision making. Let $q_a(x)$ denote the probability that the agent takes action $a$ after observing $x$. 
\begin{definition}
    An agent is a perturbed utility maximizer if, for some utility
    function $u:\mathcal A\times\Theta\to\mathbb R$ that is non-degenerate in
    relative action utilities (i.e., there exist $x,x'$ in the support of $X$
    and distinct $a,a'\in\mathcal A$ such that
    $\sum_{\theta\in\Theta}P_S(\theta\mid x)[u(a,\theta)-u(a',\theta)]
    \neq
    \sum_{\theta\in\Theta}P_S(\theta\mid x')[u(a,\theta)-u(a',\theta)]$)
    and some strictly convex regularizer
    $C:\Delta(\mathcal A)\to\mathbb R$, its choice-probability vector
    $q(x)\in\Delta(\mathcal A)$ satisfies
    \[
        q(x)
        \in
        \argmax_{\lambda \in \Delta(\mathcal A)}
        \left\{
        \sum_{a \in \mathcal A}\lambda_a
        \left(\sum_{\theta \in \Theta} P_S(\theta\mid x) u(a,\theta)\right)
        - C(\lambda)
        \right\}.
    \]
    The realized action is then drawn according to this choice rule:
    $\Pr(A=a_i\mid x,\theta)=q_{a_i}(x)$.
    \label{definition:perturbed-max}
\end{definition}
Perturbed utility maximization originated in the economics literature as a generalization of traditional choice models \citep{hofbauer2002global}. The regularizer $C$ allows this class to express a wide variety of near-utility maximizing behaviors where agents act with stochasticity, with unobserved components in their utility function, or to ensure robustness to different preferences. For example:
\begin{example}
If $C(q) = \sum_{a \in \mathcal{A}} q_a(x)\log q_a(x)$ is the negative entropy, then the choice probabilities are a softmax over expected utilities, i.e. the well-known logit model \citep{mcfadden1973conditional,train2009discrete}.
\end{example}
\begin{example}
    More generally, \cite{hofbauer2002global} show any \textit{random utility model} satisfies Definition \ref{definition:perturbed-max}. Random utility models are the most common class of discrete choice models, and model the agent as maximizing $u(a,\theta) + \epsilon_a$ for a random shock $\epsilon_a$ to account for unobserved variables impacting the agent's preferences.
\end{example}
\begin{example}
    Taking $q_0$ to be a baseline policy and $C = D_{KL}(q||q_0)$ to be the KL divergence from $q_0$ recovers an agent who trades off expected utility against deviation from a default policy (as in many post-training pipelines for LLMs \citep{ouyang2022traininglanguagemodelsfollow}). 
\end{example}
Other examples occur across the economic literature; e.g., \cite{fudenberg2015stochastic} show that Definition \ref{definition:perturbed-max} encompasses ambiguity-averse agents that maximize their worst-case expected utility over a set of different scenarios for the utility function, while \cite{fosgerau2020discrete} show that it can be used to express common models of rational inattention. Accordingly, formally grounding empirical tests of belief consistency with respect to Definition \ref{definition:perturbed-max} allows us to express properties that should relate beliefs and actions for a very broad set of agents that deliberately pursue goals in some coherent way, without requiring strict expected utility maximization. One deliberate choice is that we restrict the utility $u$ to depend only on $a$ and $\theta$, not $x$. Otherwise any action distribution $a(x)$ is trivially rationalizable (e.g., by taking $u(a,\theta,x) = 1[a = a(x)]$) without requiring the agent to have any beliefs about $\theta$ at all. Effectively, there is no observable distinction between an agent whose preferences correlate with $\theta$ via $x$ and an agent who has beliefs about $\theta$ and so we choose to interpret the component of the agent's decision-making which is able to encode predictive information about $\theta$ as part of the belief state.

\label{subsec:prospect_connection}

\subsection{Testable Implications for Belief Elicitation}
A central objective of this paper is to test whether it is possible for a model's verbalized probabilities to reflect \emph{subjective} belief states that meaningfully drive its choices. We will refer to the probability judgment \textit{elicited} from a model by some method (e.g., prompting) as $P_E(\theta\mid x)$. Our goal is to test whether it is possible for a (near) rational decision maker to hold those judgments as true subjective beliefs, $P_S(\theta|x) = P_E(\theta\mid x)$, and still take the same actions as the model. We will construct a framework to test this in the context of a specific decision problem, where we assume that the experimenter has samples from the true distribution over $(x, \theta)$. 

We gather data in two steps. First, on a sample of values of $x$, we
elicit $P_E(\theta\mid x)$ using whatever method we wish to test. In our
experiments, we prompt the model to report a probability distribution
$P_E(\theta\mid x)$ in natural language. Specifically,
for each context $x$ in our evaluation set, we query the model for probability
estimates for each possible state. For the tasks we consider, the state space is
binary ($\theta \in \{0,1\}$, indicating absence or presence of a condition), so
we elicit $p(x) := P_E(\theta = 1\mid x)$ directly as a numerical probability.
However, our methods apply to any elicitation method. We focus on
natural-language probability estimates because of their simplicity and the wide
literature investigating them. For notation, we often use
$p(x) := P_E(\theta = 1\mid x)$ or $p^*(x) := P^*(\theta = 1\mid x)$ in
shorthand for the binary case; we use $P_E(\theta \mid x)$ or
$P(\theta \mid x)$ when results generalize beyond the binary case. Second, we prompt the model with the same $x$ in a separate context and ask it
to take an action $a$. This generates a sample of the form
$(x_i, P_E(\theta = 1\mid x_i), a_i, \theta_i)$, where we observe for each
$x_i$ the model's elicited belief, its action, and the true state (which was not
revealed to the model). Finally, we apply tests examining properties that the
sample distribution should satisfy if the elicited beliefs play a meaningful
role in action. Violations quantify the extent to which elicited beliefs depart
from a useful interpretation as true subjective probabilities. We next detail
these tests.

\label{subsec:belief_elicitation}

\label{subsec:ci_test}
\paragraph{Conditional independence of actions and outcomes:} Intuitively, a valid report of a belief should summarize all of the decision-relevant information that the agent possesses about the state $\theta$. Formally, once we condition on the model's subjective belief, the agent's action $a$ should not provide any additional information about the true outcome $\theta$: $P_S$ is a sufficient statistic for the impact of the distribution of future events on the agent's choices. If a rational agent's actions do have excess correlation with the true state given their elicited $P_E$, the agent must ``know more than they say". 




\begin{proposition}[Conditional Independence under Truthful Reporting]
\label{thm:ci_rum}
For any agent satisfying Definition \ref{definition:perturbed-max}, their subjective beliefs and actions satisfy $a \perp \theta \mid P_S(\theta\mid x)$. Conversely, if for a set of elicited beliefs $P_E$, $a \not\perp \theta \mid P_E(\theta|x)$, then there is no agent satisfying Definition \ref{definition:perturbed-max} (under \emph{any} utility function $u$ and regularizer $C$) who holds subjective belief $P_S = P_E$. 
\end{proposition}
For intuition, note that in Definition \ref{definition:perturbed-max}, the expected utility term depends on $x$ only through $P_S(\theta|x)$. Once this belief term is fixed, sampling $a \sim q(x)$ conveys no additional information about $\theta$.



\begin{remark}
\label{rem:converse}
Observing $a \perp\!\!\!\perp \theta \mid P_E(\theta \mid x)$ does \emph{not} by itself suffice to prove truthful reporting by a rational agent. E.g.\, consider an agent who takes a uniformly random action independent of $x$ and reports a (independent) uniformly random value in [0,1] as $p_E(x)$. This satisfies $a \perp \theta \mid P_E(\theta|x)$ without satisfying Definition \ref{definition:perturbed-max}. 
\end{remark}
\begin{remark}
    While conditional independence can be derived from Definition \ref{definition:perturbed-max}, it can be shown to hold even more broadly. For example, we show in the appendix that it holds for agents described by prospect theoretic risk-aversion \citep{eec14168-5714-3ca8-b073-d038266f2734}. 
\end{remark}
\label{para:ci_testing}
In order to empirically test this condition, we first conduct a direct conditional-independence (CI) test of the null hypothesis
$A \perp\!\!\!\perp \theta \mid P_{E}(\theta\mid x)$ via the nonparametric CMI (Conditional Mutual Information) test for $H_0: I(A;\theta \mid p)=0$. While a binary reject/fail-to-reject CI decision is useful, it can be hard to interpret \emph{how large} a violation is in practice. Therefore, we additionally quantify the \emph{degree} of CI violation via an out-of-sample predictive-performance comparison. Concretely, we train two predictive models for the realized state $\theta$ and compare nested feature sets. \emph{Model 1} conditions on the elicited belief $p_E(x)$, while \emph{Model 2} conditions on $(A,p_E(x))$. Under the conditional-independence null, adding the action $A$ should not improve prediction of $\theta$ once we condition on $p_E$. We use cross-validated mean squared error (MSE) and report percent improvement as the relative reduction in MSE from including $A$. Significance is assessed by a bootstrap CI.

\label{subsec:cyclic_monotonicity}

\paragraph{Testing whether beliefs move actions.}
For beliefs to meaningfully describe behavior, they should be
\textit{action-guiding}: actions should shift systematically as beliefs change.
For an expected utility maximizer, expected utilities are linear in
probabilities, yielding action-switching thresholds. This deterministic
structure is too strong when choices are stochastic or perturbed, so we instead
require choice probabilities to move monotonically with beliefs. With multiple
actions, this cannot generally be imposed action-by-action, since actions
compete for probability mass. We therefore use the convex-analytic monotonicity
condition implied by perturbed utility maximization in
Definition~\ref{definition:perturbed-max}. Let $b\in\Delta(\Theta)$ be a belief
over states, $q(b):=(q_1(b),\dots,q_J(b))$ the conditional choice-probability
vector, with $q_a(b):=\Pr(a(x)=a\mid P_S(\cdot\mid x)=b)$, and
$v(b)\in\mathbb R^J$ the expected-utility index, with
$v_a(b):=\sum_{\theta\in\Theta} b_\theta u(a,\theta)$.


\begin{proposition}[Cyclic monotonicity]
\label{prop:cyclic_monotonicity}
For any agent satisfying Definition~\ref{definition:perturbed-max}, the conditional choice-probability map $q(b)$ is cyclically monotone with respect to $v(b)$. That is, for any finite collection $b_1,\dots,b_M$ with $b_{M+1}=b_1$, $\sum_{m=1}^{M} q(b_m)\cdot\big(v(b_m)-v(b_{m+1})\big)\ge 0$.
Consequently, if the observed conditional choice probabilities violate these inequalities, then the choices are not rationalizable by any agent in the model class.
\end{proposition}

Cyclic monotonicity generalizes standard monotonicity to settings with multiple
actions and choice-probability vectors. It was previously known to hold for random utility models \citep{mcfadden1990stochastic} and is easily shown to hold for the more general perturbed maximization class via convex-analytic arguments \citep{rockafellar1997convex}. In general, the condition can be
somewhat abstract, but we show it has a simpler interpretation when
$\theta\in\{0,1\}$. Define $d_a:=u(a,1)-u(a,0)$ and let
$q_a(p):=\Pr(a(x)=a\mid p(x)=p)$, where $p(x)=P_E(\theta=1\mid x)$.

\begin{proposition}[Binary-state characterization of cyclic monotonicity]
\label{prop:single-mono}
Suppose $\theta\in\{0,1\}$. For the utility-difference vector $d$ induced
by $u$, the conditional choice rule $q(p)$ is cyclically monotone with respect
to the utility index $v(p)$ if and only if, for every $p<p'$,
\[
\sum_{a=1}^{J} q_a(p)d_a
\le
\sum_{a=1}^{J} q_a(p')d_a .
\]
Moreover, under Definition~\ref{definition:perturbed-max}, the inequality is
strict whenever $q(p)\neq q(p')$.
\end{proposition}

Thus, as the elicited probability of the positive state increases, the model's
choice probabilities should move toward actions that are relatively more valuable
in state 1. To test this implication, we divide elicited beliefs into $K$
quantile bins, with bin centers $\bar p_1 \le \cdots \le \bar p_K$. For each bin
$k$, we estimate the vector of action shares
$\widehat q_k = (\widehat q_{1k},\dots,\widehat q_{Jk})$. Define
$\widehat m_k(d):=\sum_{a=1}^J \widehat q_{ak}d_a$, where
$d_a:=u(a,1)-u(a,0)$ is the relative value of action $a$ in state $1$
rather than state $0$.
By Proposition~\ref{prop:single-mono}, the data are consistent with the strict
implication of the model if there exists a utility-difference vector
$d\in\mathbb R^J$ such that $\widehat m_k(d)$ is weakly increasing in $k$, and
is strictly increasing across any adjacent bins $k,k+1$ for which
$\widehat q_k\neq \widehat q_{k+1}$. Only the relative values of $d$ matter
because each $\widehat q_k$ is a probability vector. We therefore normalize $d$
by setting one coordinate equal to $0$ and another equal to $1$. Since we do not
know which actions have the smallest and largest state-1 relative utilities, we
try all ordered pairs of distinct actions $(a^\star,b^\star)$.
Let $\mathcal S:=\{k\in\{1,\ldots,K-1\}:\widehat q_k\neq \widehat q_{k+1}\}$
denote the adjacent bin pairs for which the estimated choice-share vectors
differ. For each ordered pair $(a^\star,b^\star)$, we solve the signed-margin
linear program
\begin{equation}
\label{eq:cm-signed-lp-family}
\begin{aligned}
\mathrm{SignedLP}(a^\star,b^\star):
\qquad
\max_{d,\gamma}
\quad & \gamma \\
\text{s.t.}
\quad
& \widehat m_{k+1}(d)-\widehat m_k(d)\ge \gamma,
&& k\in\mathcal S, \\
& d_{a^\star}=0,
\qquad
d_{b^\star}=1, \\
& d\in[0,1]^J, \\
& \gamma\in\mathbb R .
\end{aligned}
\end{equation}
We summarize the signed monotonicity statistic by
$\widehat\gamma:=\max_{a^\star\neq b^\star}
\mathrm{val}\bigl(\mathrm{SignedLP}(a^\star,b^\star)\bigr)$.
This statistic is the largest uniform lower bound on the adjacent-bin increases
of the fitted utility-difference index. Hence positive values indicate strict
monotonicity with a uniform margin, zero indicates weak but not strict
monotonicity, and negative values indicate a violation of weak monotonicity.

\begin{proposition}[Equivalence with Proposition \ref{prop:single-mono}]
\label{thm:lp-equivalence}
Assume $\mathcal S\neq\varnothing$. Then $\widehat\gamma>0$ if and only if
there exists a non-constant $d\in\mathbb R^J$ such that
$\widehat m_1(d)\le \cdots \le \widehat m_K(d)$, and, moreover,
$\widehat m_k(d)<\widehat m_{k+1}(d)$ for every $k\in\mathcal S$.
Moreover, $\widehat\gamma\ge 0$ if and only if there exists a non-constant
$d\in\mathbb R^J$ such that
$\widehat m_1(d)\le \cdots \le \widehat m_K(d)$. Finally,
$\widehat\gamma\in[-1,1]$.
\end{proposition}

Proposition~\ref{thm:lp-equivalence} follows from the normalization used in the
LP and the finiteness of $\mathcal S$. If a non-constant vector $d'$ satisfies
the required weak and strict inequalities, then
$\min_{k\in\mathcal S}\{\widehat m_{k+1}(d')-\widehat m_k(d')\}>0$. The affine
rescaling $\widetilde d_a=(d'_a-\min_b d'_b)/(\max_b d'_b-\min_b d'_b)$
preserves the inequalities, lies in $[0,1]^J$, and attains both $0$ and $1$,
so it is feasible for the relevant signed-margin LP with some $\gamma>0$.
The converse is immediate from the LP constraints. The same argument with weak inequalities gives $\widehat\gamma\ge 0$ if and
only if some normalized non-constant $d$ makes the fitted index weakly
increasing. Adjacent pairs outside $\mathcal S$ impose no additional restriction,
because $\widehat q_k=\widehat q_{k+1}$ implies
$\widehat m_{k+1}(d)-\widehat m_k(d)=0$ for every $d$. Thus
$\widehat\gamma<0$ means that every normalized $d$ generates a negative
increment for at least one adjacent pair in $\mathcal S$. Finally, since
$d\in[0,1]^J$ and each $\widehat q_k$ is a probability vector,
$\widehat m_k(d)\in[0,1]$, so every adjacent increment lies in $[-1,1]$ and
therefore $\widehat\gamma\in[-1,1]$.


\textbf{Characterizing the strength of black-box belief elicitation.}
We now show that the conditional independence and monotonicity conditions constitute the full testable implications of an agent holding a belief, at least in the regime we study with black-box access and no assumptions on the agent's utility function. Formally, while we have already shown that both conditions above are necessary for Definition \ref{definition:perturbed-max} (perturbed utility maximization) to hold, we now show that they are jointly \textit{sufficient} as well: when both conditions are satisfied, the agent is \textit{observationally equivalent} to some perturbed utility maximizer who holds $P_E(\theta|x)$ as their true subjective probability.

\begin{proposition}[Joint characterization of rationalizable elicited beliefs]
\label{thm:joint_characterization}
Let $\mathbb P$ denote the observed distribution over
$(P_E(\theta\mid x),\theta,a)$, and let
$q(P_E):=\big[\Pr(a=a'\mid P_E(\theta\mid x)=P_E)\big]_{a'\in\mathcal A}$
be the induced conditional choice rule. There exists a utility function
$u:\mathcal A\times\Theta\to\mathbb R$ that is non-degenerate in relative action
utilities and an agent satisfying the perturbed-maximization representation in
Definition~\ref{definition:perturbed-max}, except with some convex, but not
necessarily strictly convex, regularizer
$C:\Delta(\mathcal A)\to\mathbb R$, whose subjective belief satisfies
$P_S(\theta\mid x)=P_E(\theta\mid x)$ and whose induced distribution over
$(P_E(\theta\mid x),\theta,a)$ equals $\mathbb P$, if and only if
$a\perp\!\!\!\perp \theta\mid P_E(\theta\mid x)$ and $q(P_E)$ is cyclically
monotone with respect to the utility index
$v_a(P_E)=\sum_{\theta\in\Theta}P_E(\theta\mid x)u(a,\theta)$.
\end{proposition}

Our theoretical results show a perhaps surprising fact: elicited beliefs have
empirically testable implications even after discarding much of the structure
underlying traditional belief elicitation for humans. We do not assume a known
utility function or require exact optimization. By contrast, the economic and
algorithmic literature on proper scoring rules and related elicitation
mechanisms typically requires substantive assumptions about incentives
\citep{savage1971elicitation,gneiting2007strictly,waggoner2014output}. Our tests
instead leverage a distinctive feature of LLMs: they can be repeatedly queried in
fresh context windows, allowing direct estimation of choice probabilities.
At the same time, our results characterize the limits of this black-box access
model. Passing our conditions only shows that, on the evaluated distribution, the
agent is observationally indistinguishable from some rational decision maker who
truly holds the reported belief. This is closer to failing to falsify the belief
report than to proving that the model ``truly'' holds that belief. For example,
the tests would not rule out deliberate misreporting or deception by a
sufficiently sophisticated agent.

\textbf{Consistency across decision tasks.}\label{subsec:prompt_consistency}
So far, we have discussed testable implications within a single decision task. Across tasks, a rational agent should separate beliefs from preferences: when multiple tasks reference the same state $\theta$, agents satisfying Definition~\ref{definition:perturbed-max} should act as if driven by a common subjective belief $p_S$, even if utilities or action sets differ. This cross-task stability---often called \textit{prize independence}---is a foundational implication of subjective probability \citep{anscombe,savage1972foundations,ronayne2022decision}. Accordingly, we compare elicited probabilities across prompts that share the same $(x,\theta)$ but differ in task framing. In the binary setting, the alternative prompt asks for $p_E(\theta\mid x)$ under an explicitly stated mean-squared-error (MSE) evaluation loss (chosen as MSE is a proper scoring rule that incentivizes telling the true belief). Let $p(x;\pi,j)$ denote the reported probability for context $x$ under prompt $\pi$ on repetition $j$, with $\pi_0$ the main diagnostic prompt and $\pi_{\mathrm{MSE}}$ the MSE prompt. We compare paired reports within the same repetition and summarize cross-task inconsistency as $\Delta_{\mathrm{MSE}}=\sqrt{\frac{1}{nr}\sum_{i=1}^{n}\sum_{j=1}^{r}\left(p(x_i;\pi_{\mathrm{MSE}},j)-p(x_i;\pi_0,j)\right)^2}.$ Larger values indicate less stable elicited beliefs across task framings. As a reference, we also report within-prompt variability for the standard prompt $\pi_0$, computed from repeated elicitations of $p(x;\pi_0,j)$.

\begin{table*}[t]
\centering
\caption{\textbf{Conditional-independence (belief sufficiency) tests.}
The first section reports kNN conditional mutual information $I(A;\theta \mid p)$ with 95\% confidence intervals. The second reports Random Forest comparisons of $\theta\sim p$ vs.\ $\theta\sim(A,p)$, summarized as percent MSE improvement with 95\% bootstrap CIs. The third reports average RF improvement by model and by dataset/domain.}
\label{tab:ci_combined_cmi_rf}

{
\footnotesize
\setlength{\tabcolsep}{3.5pt}
\renewcommand{\arraystretch}{0.92}
\setlength{\aboverulesep}{1.5pt}
\setlength{\belowrulesep}{1.5pt}
\setlength{\cmidrulekern}{0.75pt}

\begin{tabular}{l r l @{\hspace{1pt}\vrule\hspace{1pt}} r l @{\hspace{1pt}\vrule\hspace{1pt}} l r l}
\toprule
& \multicolumn{2}{c}{\textbf{kNN CMI}}
& \multicolumn{2}{c}{\textbf{RF improvement}}
& \multicolumn{3}{c}{\textbf{Average RF improvement}} \\
\cmidrule(lr){2-3} \cmidrule(lr){4-5} \cmidrule(lr){6-8}
\textbf{Dataset / Model}
& \textbf{CMI}
& \textbf{95\% CI}
& \textbf{\% Imp.}
& \textbf{95\% CI}
& \textbf{Group}
& \textbf{\% Imp.}
& \textbf{95\% CI} \\
\midrule
Heart--GPT-Min     & 0.1454 & [0.1119, 0.1789] & 18.39 & [ 0.99, 31.66] & GPT-Minimal   &  6.15 & [-1.22, 13.52] \\
Heart--GPT-High    & 0.0753 & [0.0422, 0.1085] &  9.93 & [ 2.91, 17.75] & GPT-High      &  4.89 & [ 1.25,  8.53] \\
Heart--Llama       & 0.0718 & [0.0365, 0.1070] & 14.60 & [ 6.15, 22.64] & Llama         & 15.19 & [ 6.11, 24.27] \\
Heart--DeepSeek    & 0.0675 & [0.0354, 0.0997] & 22.41 & [ 9.64, 32.98] & DeepSeek-R1   &  5.11 & [-4.83, 15.06] \\
\midrule
Cry--GPT-Min       & 0.2232 & [0.1874, 0.2589] &  6.27 & [-0.40, 11.43] & Heart Disease & 16.33 & [11.81, 20.85] \\
Cry--GPT-High      & 0.1901 & [0.1390, 0.2412] &  6.92 & [ 2.22, 10.67] & Cry           &  5.21 & [-0.03, 10.45] \\
Cry--Llama         & 0.4223 & [0.3764, 0.4681] & 11.12 & [ 2.39, 20.25] & Fever         &  2.07 & [ 0.30,  3.85] \\
Cry--DeepSeek      & 0.1745 & [0.1265, 0.2225] & -3.48 & [-6.17, -0.62] & Diabetes      &  7.73 & [-4.93, 20.38] \\
\midrule
Fever--GPT-Min     & 0.1446 & [0.1128, 0.1765] & -0.06 & [-0.24,  0.14] &               &       &                 \\
Fever--GPT-High    & 0.0944 & [0.0564, 0.1323] &  1.85 & [ 0.32,  3.27] &               &       &                 \\
Fever--Llama       & 0.3289 & [0.2893, 0.3686] &  4.95 & [-1.31, 10.65] &               &       &                 \\
Fever--DeepSeek    & 0.2060 & [0.1663, 0.2456] &  1.54 & [-3.65,  6.52] &               &       &                 \\
\midrule
Diabetes--GPT-Min  & 0.0193 & [0.0129, 0.0258] &  0.00 & [-0.00,  0.00] &               &       &                 \\
Diabetes--GPT-High & 0.0461 & [0.0182, 0.0740] &  0.85 & [-0.26,  1.91] &               &       &                 \\
Diabetes--Llama    & 0.2695 & [0.2357, 0.3033] & 30.07 & [20.56, 38.96] &               &       &                 \\
Diabetes--DeepSeek & 0.0351 & [0.0169, 0.0533] & -0.03 & [-0.10,  0.01] &               &       &                 \\
\bottomrule
\end{tabular}
}
\end{table*}

\textbf{Internal Consistency via Law of Iterated Expectation.}
\label{subsec:consistency}
So far, we have focused purely on the link between beliefs and actions. However, a parallel literature investigates whether elicited beliefs from LLMs constitute a coherent probability distribution, in the sense of adhering to probability axioms \citep{zhu2025incoherentprobabilityjudgmentslarge,herrmann2024standards}. To investigate the relationship between these two senses of coherence (decision-theoretic and probabilistic), we develop a test generalizing previous work to the settings we consider. In particular, we test adherence to the law of iterated expectation over an auxiliary variable $z$ (e.g., a patient attribute) that we can optionally reveal. Concretely, let $x$ denote the base context (with $z$ withheld), and let $\{B_1,\ldots,B_k\}$ be a partition of the support of $z$. We verify:
$P_E(\theta\mid x)
=
\sum_{j=1}^k
P_E(\theta\mid x, z \in B_j)\; P_E(z \in B_j \mid x).$ In our experiments, we construct partitions based on patient covariates and elicit three quantities (with m standing for model): $P_E(\theta\mid x)$, $P_E(\theta\mid x, z \in B_j)$ for each bin $B_j$, and $P_E(z \in B_j\mid x)$. We then measure the discrepancy: $\Delta_{\text{LIE}}(x)
=\left|
P_E(\theta\mid x)
-
\sum_{j=1}^k
P_E(\theta\mid x, z \in B_j)\; P_E(z \in B_j \mid x)
\right|.$ Large values of $\Delta_{\text{LIE}}(x)$ indicate internal inconsistencies in the model's reported beliefs, suggesting they may not reflect a coherent probabilistic reasoning process. We do not expect this sense of coherence to necessarily align with our account of valid belief elicitation. Intuitively, a decision maker can ``truthfully'' hold beliefs that are not fully consistent in a probabilistic sense (e.g., they could satisfy Definition \ref{definition:perturbed-max} without $P_S$ being a valid probability distribution). We measure  $\Delta_{\text{LIE}}$ to provide an empirical comparison of where these senses of coherence overlap and differ.


\section{Experimental Setup} 
\label{subsec:implementation}
We apply our methodology to four medical diagnosis tasks: two using real-world datasets (structural heart disease and diabetes) and two using expert-constructed Bayesian networks (pediatric medicine). We evaluate each dataset on the following language models: GPT-5 Thinking High Reasoning and Minimal Reasoning Models \citep{singh2025openaigpt5card}, Deepseek R1 671B \citep{Guo_2025}, and Llama-4 Scout 17B \citep{meta2024Llama4}. These models allow us to examine a range of both reasoning levels and model sizes, encompassing both SOTA frontier and open-source models. For each metric, we report 95$\%$ bootstrapped CIs. Additional details regarding prompts, datasets and other implementation details can be found in the Appendix and released code. For each dataset, we sample 200 covariate-outcome pairings with five repetitions.

\textbf{Datasets.} We use publicly available datasets: (1) electrocardiograms with demographic data and structural heart disease labels \citep{Elias_Finer_2025}, and (2) diabetic patient records containing fields such as exercise and glucose levels \citep{cdc_2017}. Ground-truth $p^*(x)$ are computed as the fraction of positive diagnoses falling into covariate strata (always at least 100 patients). We also evaluate on distributions from two Bayes nets constructed by a leading expert in pediatric medicine.  Pending completion of a data-sharing agreement, we plan to have these permissions by the camera-ready. We use networks in pediatric medicine constructed separately for chief complaints of fever and crying, respectively. The fever network contains demographic and symptom covariates (e.g., jaundice). The crying network covers causes such as colic and gas pain with behavioral/physical covariates.

\textbf{Decision task.} For the medical diagnosis tasks that we consider, the state space is binary ($\theta \in \{0,1\}$ indicating
absence or presence of a condition). We prompt the language model with each context $x_i$ and observe its chosen action $a_i \in \mathcal{A}$. For our diagnostic tasks, the action space consists of three options: diagnose the patient as having the condition, diagnose the patient as not having the condition, or `defer", refusing to make any diagnosis. The model is instructed to first choose if it feels capable of making a diagnostic decision, and then asked which decision the model would make if it had to make a decision with only binary $a = \{1,0\}$ options presented. These questions can then translate our decision into one of the three options. Deferral is treated as a distinct action with its own cost. 

\textbf{Prompting.} When eliciting either beliefs or decisions, we provide the model with the clinical evidence $x$ and specify the relevant outcome $\theta$, then instruct it to return a numerical probability estimate. For example, we might ask for the probability of structural heart disease given that a patient is over 50 and has a particular ECG abnormality, and separately ask the model what diagnostic action it would take from the same evidence. Throughout, we elicit beliefs as $p(x):=P_E(\theta=1\mid x)$. Unless otherwise noted, we use a standard version of these prompts without additional instructions.

\section{Results}
\paragraph{Conditional Independence Test.}
\label{subsec:results_ci}

We first test whether the elicited belief is \emph{decision-sufficient} for the
model's actions, as implied by Proposition~\ref{thm:ci_rum}. Under truthful
reporting with outcome-dependent loss and exogenous decision noise, the model's
action $A$ should add no information about the realized outcome $\theta$ once we
condition on the elicited belief, equivalently $H_0:I(A;\theta\mid p)=0$. The
left side of Table~\ref{tab:ci_combined_cmi_rf} reports a kNN ($k=3$)
conditional mutual information (CMI) test; all model--dataset pairs reject the
null, with 95\% bootstrap CIs (500 resamples) strictly above zero. Appendix
Table~\ref{tab:cmi_raw_iso} shows the same conclusion for isotonic-calibrated
beliefs, indicating that calibration does not restore decision-sufficiency.
To gauge the size of this residual dependence, the right side of
Table~\ref{tab:ci_combined_cmi_rf} reports a cross-validated Random Forest
comparison, with parameter details in the Appendix to prevent data contamination
and control overfitting. We compare a baseline model predicting the realized
state from the elicited belief alone, $\theta\sim p$, against an augmented model
that also includes the action, $\theta\sim(A,p)$. Under the
conditional-independence null, adding $A$ should not reduce out-of-sample MSE
once $p$ is included. Empirically, MSE reductions are heterogeneous by both
domain and model, with the clearest reductions for the Heart Disease domain and
for Llama across domains. These results suggest that $p$ captures most---but not
all---decision-relevant signal in many settings; in cases such as Llama and Heart
Disease, elicited probabilities do not appear to serve as sufficient statistics
for decisions, motivating task- and model-specific validation.


\begin{table}[t]
\centering
\caption{
Signed-margin monotonicity statistic $\widehat{\gamma}$ by dataset and model.
Positive values indicate strict monotonicity with a uniform margin, values equal
to zero indicate weak but not strict monotonicity, and negative values indicate
violations of weak monotonicity. Brackets report 95\% confidence intervals.
The right panel reports unweighted averages by model and by domain.
}
\label{tab:signed-monotonicity}
\begin{minipage}[t]{0.31\linewidth}
\centering
\footnotesize
{
\setlength{\tabcolsep}{2pt}
\renewcommand{\arraystretch}{0.9}
\begin{tabular}{l r}
\toprule
\textbf{Dataset / Model} & \textbf{$\widehat{\gamma}$ [95\% CI]} \\
\midrule
Heart--GPT-Min     & 0.104 [0.012, 0.168] \\
Heart--GPT-High    & 0.059 [0.000, 0.160] \\
Heart--Llama       & 0.042 [0.005, 0.068] \\
Heart--DeepSeek    & 0.192 [0.117, 0.205] \\
Cry--GPT-Min       & 0.050 [0.001, 0.117] \\
Cry--GPT-High      & 0.059 [0.003, 0.120] \\
Cry--Llama         & 0.266 [0.217, 0.317] \\
Cry--DeepSeek      & 0.097 [0.034, 0.142] \\
\bottomrule
\end{tabular}}
\end{minipage}
\hfill
\begin{minipage}[t]{0.31\linewidth}
\centering
\footnotesize
{
\setlength{\tabcolsep}{2pt}
\renewcommand{\arraystretch}{0.9}
\begin{tabular}{l r}
\toprule
\textbf{Dataset / Model} & \textbf{$\widehat{\gamma}$ [95\% CI]} \\
\midrule
Fever--GPT-Min     & 0.053 [0.000, 0.082] \\
Fever--GPT-High    & 0.103 [0.001, 0.194] \\
Fever--Llama       & -0.011 [-0.018, -0.005] \\
Fever--DeepSeek    & 0.148 [0.078, 0.175] \\
Diabetes--GPT-Min  & 0.000 [0.000, 0.032] \\
Diabetes--GPT-High & 0.031 [0.000, 0.049] \\
Diabetes--Llama    & 0.000 [0.000, 0.003] \\
Diabetes--DeepSeek & 0.010 [0.000, 0.028] \\
\bottomrule
\end{tabular}}
\end{minipage}
\hfill
\begin{minipage}[t]{0.31\linewidth}
\centering
\footnotesize
{
\setlength{\tabcolsep}{4pt}
\renewcommand{\arraystretch}{0.9}
\begin{tabular}{l r}
\toprule
\textbf{Average} & \textbf{$\widehat{\gamma}$} \\
\midrule
GPT-Min          & 0.052 \\
GPT-High         & 0.063 \\
Llama            & 0.074 \\
DeepSeek         & 0.112 \\
Heart            & 0.099 \\
Cry              & 0.118 \\
Fever            & 0.073 \\
Diabetes         & 0.010 \\
\bottomrule
\end{tabular}}
\end{minipage}
\end{table}
\textbf{Cyclic Monotonicity Test.}
Table~\ref{tab:signed-monotonicity} reports the signed cyclic-monotonicity
margin $\widehat{\gamma}$ for each dataset/model pair, computed by binning
elicited probabilities into five quantiles and estimating choice shares within
bins. Under this formulation, $\widehat{\gamma}$ is the largest uniform lower
bound on adjacent-bin increases in the fitted utility-difference index: positive
values indicate strict monotonicity, zero indicates weak but not strict
monotonicity, and negative values indicate a violation of weak monotonicity. Results vary across domains and models. In 9 of 16 dataset/model combinations,
the 95\% confidence interval lies strictly above zero, indicating strict cyclic
monotonicity. Six combinations have non-negative intervals that include zero,
consistent with weak monotonicity but not a positive strict margin. One
combination, Fever--Llama, has an interval strictly below zero, indicating a
weak-monotonicity violation. Domain effects are pronounced: Diabetes performs
weakest, with no model showing a strictly positive interval and an average
signed margin near zero. Across models, DeepSeek has the largest average signed
margin. 



\textbf{Prompt Consistency.}
In Table~\ref{tab:prompt_consistency_rmse} (Appendix), we report the standard deviation of elicited beliefs under the standard prompt and the RMSE deviation between elicited beliefs under the standard prompt and an alternative MSE-framed prompt. We find that for GPT-Min, GPT-High, and DeepSeek, the MSE-framed prompt produces deviations that are comparable to within-prompt variability under repeated standard prompting. For Llama, however, the MSE framing induces substantially larger differences in elicited probabilities. Overall, GPT-High displays the lowest cross-prompt variability, while Llama exhibits the largest sensitivity to the alternative task framing.


\textbf{Internal Consistency Test.} As in Section~\ref{subsec:consistency}, we
elicit (i) the model's marginal belief $p(x)$ that a patient has the condition
and (ii) the beliefs needed for the LIE decomposition over an auxiliary
``next-state'' variable $z$. We measure probabilistic coherence using the Law of
Iterated Expectation error $\Delta_{\text{LIE}}(x)$, summarized in
Figure~\ref{fig:consistency} by the median normalized ratio
$\Delta_{\text{LIE}}(x)/p(x)$ for robustness when $p(x)$ is near zero. As a
baseline, we compute the analogous ratio using cross-validated random forest
predictors trained on ground-truth $p^*(x)$ and $p^*(x,z)$. Across datasets,
LLMs typically show substantially larger inconsistency than this baseline, with
no model consistently dominating. This suggests that probabilistic coherence is
distinct from decision-consistency: models that perform well on the earlier
decision-consistency tests do not necessarily perform well here.

\begin{figure}
    \centering
    \includegraphics[scale=.15]{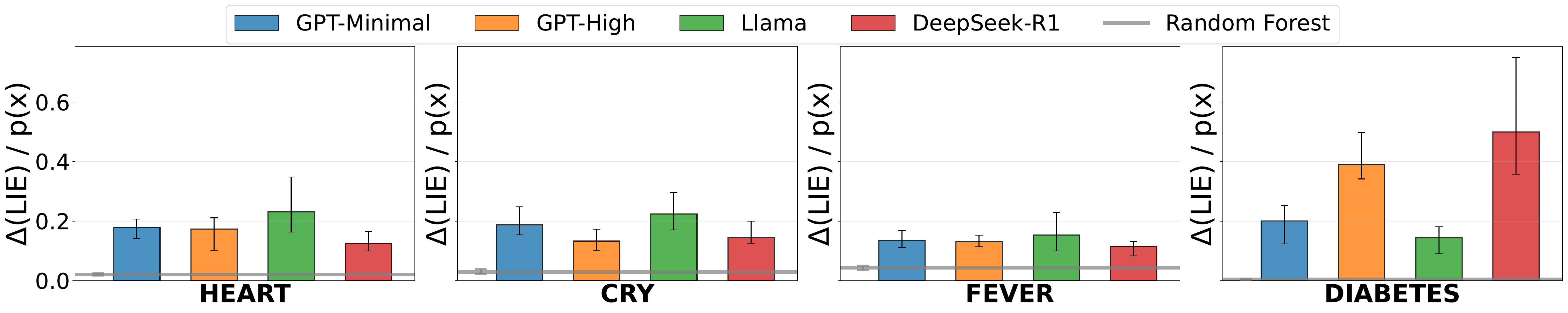}
    \caption{Internal Consistency Error via Law of Iterated Expectation (LIE): Here we plot the 95$\%$ CI for the median quantity  $\Delta_{\text{LIE}}(x)$/$p(x)$  (see Section~\ref{subsec:consistency}) by dataset(subplots) and model(bars). A cross-validated random forest model that is separately trained on ground truth $p^*(x)$ and $p^*(x,z)$ (also conditional on next state distribution z) values is used as a baseline.}
    \label{fig:consistency}
\end{figure}

\section{Discussion and Conclusions}

We develop a black-box decision-theoretic framework for testing whether elicited
LLM probabilities act as subjective beliefs guiding decisions. Because the
framework allows unrestricted utilities and stochastic choice, test failures
indicate not merely deviations from expected utility, but failures of elicited
probabilities to summarize decision-relevant information. Theoretically,
Proposition~\ref{thm:joint_characterization} shows that conditional independence
and monotonicity are necessary and sufficient for the observed distribution to be
rationalized by a perturbed utility maximizer using the elicited belief as its
subjective belief. Empirically, across medical diagnosis tasks and LLMs, elicited
beliefs often fail the sufficiency test, although choices frequently remain
monotone in elicited risk. For the strongest models, however, these failures are
small: elicited probabilities are imperfect but behaviorally informative
summaries of decision-relevant information. Thus, elicited probabilities can help
interpret model decisions, but should be empirically validated rather than
assumed to faithfully represent model beliefs.

\textbf{Acknowledgments:} Research was sponsored by the Army Research Office and was accomplished under Grant Number W911NF-25-1-0281. The views and conclusions contained in this document are those of the authors and should not be interpreted as representing the official policies, either expressed or implied, of the Army Research Office or the U.S. Government. The U.S. Government is authorized to reproduce and distribute reprints for Government purposes notwithstanding any copyright notation herein.



\clearpage
\bibliography{iclr2026_conference}
\bibliographystyle{iclr2026_conference}
\newpage
\appendix
\section*{Appendix}
\section{Proofs}

\subsection{Proof of Proposition~\ref{thm:ci_rum}}

\begin{proof}
Let $B:=P_S(\theta\mid x)$ denote the agent's subjective posterior belief,
viewed as a random element of $\Delta(\Theta)$. For each belief value $b$, define
the expected-utility index
\[
v_a(b):=\sum_{\theta\in\Theta} b_\theta u(a,\theta).
\]
By Definition~\ref{definition:perturbed-max}, conditional on $B=b$, the agent's
choice-probability vector is
\[
q(b)\in
\argmax_{\lambda\in\Delta(\mathcal A)}
\left\{
\sum_{a\in\mathcal A}\lambda_a v_a(b)-C(\lambda)
\right\}.
\]
Thus the conditional distribution of the action depends on the context $x$ only
through the subjective belief $B$.

Fix any action $a\in\mathcal A$. Since
$\Pr(A=a\mid x,\theta)=q_a(B)$, we have
\[
\Pr(A=a\mid B,\theta)
=
\mathbb E[\Pr(A=a\mid x,\theta)\mid B,\theta]
=
\mathbb E[q_a(B)\mid B,\theta]
=
q_a(B).
\]
Similarly, $\Pr(A=a\mid B)=q_a(B)$. Therefore
\[
\Pr(A=a\mid B,\theta)=\Pr(A=a\mid B)
\]
for every action $a$, which implies
\[
A\perp\!\!\!\perp \theta \mid B.
\]
Equivalently,
\[
a\perp\!\!\!\perp \theta\mid P_S(\theta\mid x).
\]

For the converse statement, suppose the elicited belief is the agent's subjective
belief, so that $P_E(\theta\mid x)=P_S(\theta\mid x)$ almost surely. Then the
argument above implies
\[
a\perp\!\!\!\perp \theta\mid P_E(\theta\mid x).
\]
Hence, if the observed joint distribution satisfies
\[
a\not\!\perp\!\!\!\perp \theta\mid P_E(\theta\mid x),
\]
then no agent satisfying Definition~\ref{definition:perturbed-max}, under any
utility function $u$ and regularizer $C$, can have $P_S=P_E$ while generating
the observed actions.
\end{proof}

\subsubsection{Conditional independence under prospect-theoretic risk aversion}

\begin{proof}
The conditional-independence conclusion does not rely on expected-utility
linearity. It only requires that, conditional on the agent's subjective belief,
the distribution over actions is fixed.

Let $B:=P_S(\theta\mid x)$ denote the agent's subjective posterior belief. Suppose
the agent evaluates each action using a prospect-theoretic index
$V_a(B)$, where $V_a(B)$ may incorporate a value function, reference dependence,
loss aversion, and probability weighting, but depends on the context $x$ only
through $B$. For example, one may write
$V_a(B)=\sum_{\theta\in\Theta}\omega_\theta(B)\phi(u(a,\theta)-r_a)$,
where $\phi$ is a prospect-theoretic value function, $\omega(B)$ is a
deterministic probability-weighting transformation of the belief $B$, and
$r_a$ is a reference point. More generally, the argument only requires that
$V_a(B)$ is measurable with respect to $B$.

The prospect-theoretic perturbed choice rule is then
\[
    q(B)
    \in
    \argmax_{\lambda\in\Delta(\mathcal A)}
    \left\{
    \sum_{a\in\mathcal A}\lambda_a V_a(B)-C(\lambda)
    \right\}.
\]
Thus, conditional on $B$, the choice-probability vector is a deterministic
function of $B$. In particular, for every action $a\in\mathcal A$,
$\Pr(A=a\mid x,\theta)=q_a(B)$. Therefore,
\[
\Pr(A=a\mid B,\theta)
=
\mathbb E[\Pr(A=a\mid x,\theta)\mid B,\theta]
=
\mathbb E[q_a(B)\mid B,\theta]
=
q_a(B).
\]
Similarly, $\Pr(A=a\mid B)=q_a(B)$. Hence
$\Pr(A=a\mid B,\theta)=\Pr(A=a\mid B)$ for every action $a$, which implies
$A\perp\!\!\!\perp \theta\mid B$.

Consequently, if the elicited belief equals the subjective belief, so that
$P_E(\theta\mid x)=P_S(\theta\mid x)$ almost surely, then
$A\perp\!\!\!\perp \theta\mid P_E(\theta\mid x)$ also holds. Thus a violation of
this conditional-independence restriction cannot be explained merely by
prospect-theoretic risk aversion, provided the prospect-theoretic evaluation
depends on the context only through the agent's subjective belief.
\end{proof}
\subsection{Proof of Proposition~\ref{prop:cyclic_monotonicity}}

\begin{proof}
Let $b_1,\dots,b_M$ be any finite collection of beliefs, and set
$b_{M+1}=b_1$. Write
\[
v_m:=v(b_m),
\qquad
q_m:=q(b_m).
\]
By Definition~\ref{definition:perturbed-max}, for each $m$,
\[
q_m\in
\argmax_{\lambda\in\Delta(\mathcal A)}
\left\{
\lambda\cdot v_m-C(\lambda)
\right\}.
\]
Therefore, comparing $q_m$ to $q_{m-1}$ at the utility index $v_m$, with indices
interpreted cyclically, we have
\[
q_m\cdot v_m-C(q_m)
\ge
q_{m-1}\cdot v_m-C(q_{m-1}).
\]
Rearranging,
\[
v_m\cdot(q_m-q_{m-1})
\ge
C(q_m)-C(q_{m-1}).
\]
Summing over $m=1,\dots,M$ gives
\[
\sum_{m=1}^M v_m\cdot(q_m-q_{m-1})
\ge
\sum_{m=1}^M \bigl(C(q_m)-C(q_{m-1})\bigr).
\]
The right-hand side telescopes to zero, so
\[
\sum_{m=1}^M v_m\cdot(q_m-q_{m-1})\ge 0.
\]
Reindexing the second term yields
\[
\sum_{m=1}^M q_m\cdot(v_m-v_{m+1})\ge 0.
\]
Substituting back $q_m=q(b_m)$ and $v_m=v(b_m)$ gives
\[
\sum_{m=1}^{M}
q(b_m)\cdot\bigl(v(b_m)-v(b_{m+1})\bigr)
\ge 0,
\]
which is cyclic monotonicity.

The final claim follows by contrapositive: if the observed conditional
choice-probability map violates one of these inequalities, then it cannot be
generated by any agent satisfying Definition~\ref{definition:perturbed-max}
using the elicited probabilities as the operative belief index.
\end{proof}

\subsection{Proof of Proposition~\ref{prop:single-mono}}
\begin{proof}
For a binary state, the expected-utility index for action $a$ at belief $p$ is
\[
v_a(p)
=
(1-p)u(a,0)+p u(a,1)
=
u(a,0)+p d_a ,
\]
where $d_a:=u(a,1)-u(a,0)$. Hence, for any $p,p'$,
\[
v(p')-v(p)=(p'-p)d .
\]
Define
\[
m(p):=q(p)\cdot d
=
\sum_{a=1}^{J}q_a(p)d_a .
\]

We first show that cyclic monotonicity with respect to $v(p)$ is equivalent to
weak monotonicity of $m(\cdot)$. Cyclic monotonicity requires that, for every
finite cycle $p_1,\ldots,p_M$ with $p_{M+1}=p_1$,
\[
\sum_{k=1}^{M}
q(p_k)\cdot\bigl(v(p_k)-v(p_{k+1})\bigr)
\ge 0 .
\]
Using
\[
v(p_k)-v(p_{k+1})=(p_k-p_{k+1})d,
\]
this condition is equivalent to
\[
\sum_{k=1}^{M}m(p_k)(p_k-p_{k+1})\ge 0 .
\]

We now prove that this one-dimensional cycle condition is equivalent to weak
monotonicity of $m$. First suppose the cycle condition holds. Consider any
$p<p'$ and apply the condition to the two-point cycle $(p,p')$. Then
\[
m(p)(p-p')+m(p')(p'-p)\ge 0 .
\]
Equivalently,
\[
(p'-p)\bigl(m(p')-m(p)\bigr)\ge 0 .
\]
Since $p'-p>0$, this implies
\[
m(p)\le m(p') .
\]
Thus $m$ is weakly increasing.

Conversely, suppose $m$ is weakly increasing. Let
$p_1,\ldots,p_M$ be any finite cycle with $p_{M+1}=p_1$. Since only finitely
many belief values appear in the cycle, list their distinct values as
\[
r_1<r_2<\cdots<r_L .
\]
Define a function $M$ on these grid points by setting $M(r_1)=0$ and, for
$i=1,\ldots,L-1$,
\[
M(r_{i+1})-M(r_i)
=
m(r_i)(r_{i+1}-r_i).
\]
Extend $M$ linearly on each interval $[r_i,r_{i+1}]$. Because $m$ is weakly
increasing, the slopes of this piecewise-linear function are weakly increasing.
Therefore $M$ is convex.

Moreover, for every grid point $r_i$, the number $m(r_i)$ is a subgradient of
$M$ at $r_i$. Indeed, the left slope at $r_i$ is weakly below $m(r_i)$ and the
right slope at $r_i$ is weakly above $m(r_i)$, with the natural one-sided
interpretation at the endpoints. Hence, for every $k$,
\[
M(p_{k+1})
\ge
M(p_k)+m(p_k)(p_{k+1}-p_k).
\]
Rearranging gives
\[
m(p_k)(p_k-p_{k+1})
\ge
M(p_k)-M(p_{k+1}).
\]
Summing over $k=1,\ldots,M$ yields
\[
\sum_{k=1}^{M}m(p_k)(p_k-p_{k+1})
\ge
\sum_{k=1}^{M}\bigl(M(p_k)-M(p_{k+1})\bigr)
=0,
\]
where the final equality follows from $p_{M+1}=p_1$. Thus the one-dimensional
cycle condition holds. Therefore, cyclic monotonicity with respect to $v(p)$
holds if and only if $m$ is weakly increasing.

It follows that cyclic monotonicity with respect to $v(p)$ holds if and only
if, for every $p<p'$,
\[
m(p)\le m(p') .
\]
Expanding the definition of $m$ gives
\[
\sum_{a=1}^{J}q_a(p)d_a
\le
\sum_{a=1}^{J}q_a(p')d_a .
\]
This proves the equivalence claimed in the first part of the proposition.

It remains to prove the strict part under
Definition~\ref{definition:perturbed-max}. Suppose the agent satisfies
Definition~\ref{definition:perturbed-max}. Let
\[
q:=q(p),
\qquad
q':=q(p').
\]
Optimality of $q$ at $v(p)$ and of $q'$ at $v(p')$ implies
\[
q\cdot v(p)-C(q)
\ge
q'\cdot v(p)-C(q')
\]
and
\[
q'\cdot v(p')-C(q')
\ge
q\cdot v(p')-C(q).
\]
Adding these two inequalities gives
\[
q\cdot v(p)+q'\cdot v(p')
\ge
q'\cdot v(p)+q\cdot v(p') .
\]
Rearranging,
\[
(q'-q)\cdot\bigl(v(p')-v(p)\bigr)\ge 0 .
\]
Since
\[
v(p')-v(p)=(p'-p)d
\]
and $p'-p>0$, this implies
\[
(q'-q)\cdot d\ge 0,
\]
or equivalently
\[
q'\cdot d\ge q\cdot d .
\]

Now suppose $q\neq q'$. For each fixed $v$, the objective
\[
\lambda\mapsto \lambda\cdot v-C(\lambda)
\]
is strictly concave on $\Delta(\mathcal A)$ because $C$ is strictly convex.
Therefore the optimizer at each fixed $v$ is unique. Since $q$ is the optimizer
at $v(p)$ and $q'\neq q$, the first optimality inequality is strict:
\[
q\cdot v(p)-C(q)
>
q'\cdot v(p)-C(q').
\]
Similarly, since $q'$ is the optimizer at $v(p')$ and $q\neq q'$, the second
optimality inequality is strict:
\[
q'\cdot v(p')-C(q')
>
q\cdot v(p')-C(q).
\]
Adding the strict inequalities yields
\[
(q'-q)\cdot\bigl(v(p')-v(p)\bigr)>0 .
\]
Again using $v(p')-v(p)=(p'-p)d$ and $p'-p>0$, we obtain
\[
q'\cdot d>q\cdot d .
\]
Thus the monotonicity inequality is strict whenever $q(p)\neq q(p')$.
\end{proof}

\subsection{Proof of Proposition~\ref{thm:lp-equivalence}}

\begin{proof}
We first prove the strict part. Suppose there exists a non-constant vector
$d'\in\mathbb R^J$ such that
$\widehat m_1(d')\le \cdots \le \widehat m_K(d')$ and
$\widehat m_k(d')<\widehat m_{k+1}(d')$ for every $k\in\mathcal S$. Since
$\mathcal S$ is finite and nonempty,
\[
\delta
:=
\min_{k\in\mathcal S}
\{\widehat m_{k+1}(d')-\widehat m_k(d')\}
>0.
\]
Let $\alpha:=\min_a d'_a$ and $\beta:=\max_a d'_a$. Since $d'$ is
non-constant, $\beta>\alpha$. Define
\[
\widetilde d_a
:=
\frac{d'_a-\alpha}{\beta-\alpha}.
\]
Then $\widetilde d\in[0,1]^J$, and there exist actions $a^\star,b^\star$ such
that $\widetilde d_{a^\star}=0$ and $\widetilde d_{b^\star}=1$. Because each
$\widehat q_k$ is a probability vector,
\[
\widehat m_k(\widetilde d)
=
\frac{\widehat m_k(d')-\alpha}{\beta-\alpha}.
\]
Thus the affine rescaling preserves all weak and strict inequalities, and the
minimum increment over $\mathcal S$ becomes
$\delta/(\beta-\alpha)>0$. Therefore $(\widetilde d,\gamma)$ is feasible for
$\mathrm{SignedLP}(a^\star,b^\star)$ for any
$0<\gamma\le \delta/(\beta-\alpha)$, so $\widehat\gamma>0$.

Conversely, suppose $\widehat\gamma>0$. Then for some ordered pair
$(a^\star,b^\star)$ there exists a feasible solution $(d,\gamma)$ to
$\mathrm{SignedLP}(a^\star,b^\star)$ with $\gamma>0$. Hence
\[
\widehat m_{k+1}(d)-\widehat m_k(d)\ge \gamma>0,
\qquad k\in\mathcal S.
\]
For $k\notin\mathcal S$, we have $\widehat q_k=\widehat q_{k+1}$, so
$\widehat m_{k+1}(d)-\widehat m_k(d)=0$. Therefore
$\widehat m_1(d)\le \cdots \le \widehat m_K(d)$, with strict inequality for
every $k\in\mathcal S$. Moreover, $d$ is non-constant because
$d_{a^\star}=0$ and $d_{b^\star}=1$.

The weak part is analogous. If there exists a non-constant $d'\in\mathbb R^J$
such that $\widehat m_1(d')\le \cdots \le \widehat m_K(d')$, then the same
affine rescaling produces a normalized vector $\widetilde d\in[0,1]^J$ attaining
both $0$ and $1$. For the corresponding ordered pair $(a^\star,b^\star)$,
\[
\widehat m_{k+1}(\widetilde d)-\widehat m_k(\widetilde d)\ge 0,
\qquad k\in\mathcal S,
\]
so $(\widetilde d,0)$ is feasible for
$\mathrm{SignedLP}(a^\star,b^\star)$ and hence $\widehat\gamma\ge 0$.
Conversely, if $\widehat\gamma\ge 0$, then some feasible solution satisfies
$\gamma\ge 0$, which implies weakly nonnegative increments on $\mathcal S$.
The increments outside $\mathcal S$ are identically zero, so
$\widehat m_1(d)\le \cdots \le \widehat m_K(d)$. Again, $d$ is non-constant
because it attains both $0$ and $1$.

It remains to show that $\widehat\gamma\in[-1,1]$. For any feasible
$d\in[0,1]^J$, each $\widehat m_k(d)$ lies in $[0,1]$ because
$\widehat q_k$ is a probability vector. Hence every adjacent increment
$\widehat m_{k+1}(d)-\widehat m_k(d)$ lies in $[-1,1]$, so any feasible
$\gamma$ satisfying the signed-margin constraints must satisfy $\gamma\le 1$.
For the lower bound, each signed LP is feasible with $\gamma=-1$, since all
adjacent increments are at least $-1$. Therefore each LP value lies in
$[-1,1]$, and taking the maximum over ordered pairs gives
$\widehat\gamma\in[-1,1]$.
\end{proof}

\subsection{Proof of Proposition~\ref{thm:joint_characterization}}

\begin{proof}
Let $B:=P_E(\theta\mid x)$ denote the elicited belief, viewed as a random
element of $\Delta(\Theta)$. For each belief value $b$, define
$q(b):=\bigl(\Pr_{\mathbb P}(a=a'\mid B=b)\big)_{a'\in\mathcal A}$.
For a utility function $u:\mathcal A\times\Theta\to\mathbb R$, define the
expected-utility index $v(b)\in\mathbb R^{|\mathcal A|}$ by
$v_a(b):=\sum_{\theta\in\Theta}b_\theta u(a,\theta)$.

We prove both directions.

First suppose there exists an agent satisfying
Definition~\ref{definition:perturbed-max}, with the strict-convexity requirement
on $C$ relaxed to convexity, whose subjective belief satisfies $P_S=B$ and
whose induced distribution over $(B,\theta,a)$ equals $\mathbb P$. The proof of
Proposition~\ref{thm:ci_rum} uses only that the agent's choice probabilities
depend on $x$ through its subjective belief, not strict convexity of $C$.
Therefore Proposition~\ref{thm:ci_rum} applies unchanged and implies
$a\perp\!\!\!\perp\theta\mid B$. The proof of
Proposition~\ref{prop:cyclic_monotonicity} uses only convexity of $C$, not
strict convexity, so it also applies unchanged under this relaxed version of
Definition~\ref{definition:perturbed-max}. Therefore
Proposition~\ref{prop:cyclic_monotonicity} implies that $q(B)$ is cyclically
monotone with respect to the utility index $v(B)$. Thus the two stated
conditions are necessary.

Conversely, suppose $a\perp\!\!\!\perp\theta\mid B$ under the observed
distribution $\mathbb P$, and suppose that $q(B)$ is cyclically monotone with
respect to $v_a(B)=\sum_{\theta\in\Theta}B_\theta u(a,\theta)$ for some nonzero
utility function $u:\mathcal A\times\Theta\to\mathbb R$.

We first construct a convex-perturbation representation of the conditional
choice rule $q$. By Rockafellar's cyclic-monotonicity theorem \citep{rockafellar1997convex}, cyclic
monotonicity of the relation $\{(v(B),q(B))\}$ implies that there exists a
convex potential $W:\mathbb R^{|\mathcal A|}\to\mathbb R\cup\{+\infty\}$ such
that $q(B)\in\partial W(v(B))$ for every elicited belief value $B$ in the
support of the observed distribution. By Fenchel duality,
$q(B)\in\partial W(v(B))$ if and only if $v(B)\in\partial W^\ast(q(B))$, where
$W^\ast$ is the convex conjugate of $W$. Equivalently,
\[
q(B)\in
\argmax_{\lambda\in\Delta(\mathcal A)}
\{\lambda\cdot v(B)-C(\lambda)\},
\]
where $C$ is the restriction of $W^\ast$ to $\Delta(\mathcal A)$. The function
$C$ is convex, but need not be strictly convex. Therefore the cyclic-monotonicity
hypothesis gives the perturbed-maximization representation
\[
q(B)
\in
\argmax_{\lambda\in\Delta(\mathcal A)}
\left\{
\sum_{a\in\mathcal A}\lambda_a
\left(
\sum_{\theta\in\Theta}B_\theta u(a,\theta)
\right)
-
C(\lambda)
\right\}.
\]

Now construct an agent as follows. The agent faces the same marginal
distribution over $(B,\theta)$ as in $\mathbb P$. Its subjective belief is
defined to be $P_S(\theta\mid x)=B=P_E(\theta\mid x)$. Conditional on $B$, the
agent uses the choice-probability vector $q(B)$ obtained above and draws its
realized action according to $\Pr_{\mathrm{agent}}(a\mid B)=q_a(B)$. By
construction, this agent satisfies Definition~\ref{definition:perturbed-max}
with nonzero utility function $u$ and with a convex, but not necessarily
strictly convex, regularizer $C$.

It remains to verify that the constructed agent reproduces the full observed
joint distribution over $(B,\theta,a)$. For any $B$, $\theta$, and $a$,
\[
\Pr_{\mathrm{agent}}(B,\theta,a)
=
\Pr_{\mathrm{agent}}(B,\theta)
\Pr_{\mathrm{agent}}(a\mid B,\theta).
\]
By construction,
\[
\Pr_{\mathrm{agent}}(B,\theta)=\Pr_{\mathbb P}(B,\theta).
\]
Moreover, the constructed action rule depends only on $B$, so
\[
\Pr_{\mathrm{agent}}(a\mid B,\theta)
=
\Pr_{\mathrm{agent}}(a\mid B)
=
q_a(B).
\]
By the conditional-independence assumption in the observed distribution,
\[
\Pr_{\mathbb P}(a\mid B,\theta)
=
\Pr_{\mathbb P}(a\mid B)
=
q_a(B).
\]
Therefore,
\[
\Pr_{\mathrm{agent}}(B,\theta,a)
=
\Pr_{\mathbb P}(B,\theta)q_a(B)
=
\Pr_{\mathbb P}(B,\theta)
\Pr_{\mathbb P}(a\mid B,\theta)
=
\Pr_{\mathbb P}(B,\theta,a).
\]
Hence the constructed agent has subjective belief $P_S=P_E$, satisfies the
perturbed-maximization representation with nonzero utility $u$ and convex, but
not necessarily strictly convex, regularizer $C$, and induces exactly the
observed distribution over $(P_E(\theta\mid x),\theta,a)$.

This proves the desired equivalence.
\end{proof}
\newpage
\section{Prompts}
\label{appendix:prompts}

This section provides the complete prompt templates used in our experiments. All prompts follow a consistent structure where clinical evidence is converted to natural language descriptions. We use placeholder notation: \texttt{<patient\_description>} denotes the natural language description of patient evidence (e.g., ``is male, is in the 50-64 age group, and QRS duration is prolonged''), and \texttt{<clinical\_question>} denotes the condition-specific question (e.g., ``have moderate or greater structural heart disease'').

\subsection{Probability Elicitation Prompts}
\label{appendix:prob_prompts}

\subsubsection{Standard Probability Elicitation (Prompt $\pi_0$)}
\label{appendix:standard_prob}

The standard prompt elicits $P_E(\theta \mid x)$ without any additional instructions about scoring rules or reasoning strategies.

\begin{quote}
\small
\texttt{There is a patient who <patient\_description>. What is the probability that they <clinical\_question>? Return probabilities for: No and Yes.}

\texttt{Respond with exactly 2 lines, one per label, and no extra text.}\\
\texttt{Each line MUST be exactly: '<label>: <number between 0 and 1>'.}\\
\texttt{Use these labels in this order: [No, Yes]}\\
\texttt{Example:}\\
\texttt{No: 0.50}\\
\texttt{Yes: 0.50}
\end{quote}

\subsubsection{MSE Scoring Rule Prompt ($\pi_{\text{MSE}}$)}
\label{appendix:mse_prompt}

This prompt includes an explicit statement that responses will be evaluated under the Mean Squared Error (Brier) scoring rule, which is strictly proper.

\begin{quote}
\small
\texttt{IMPORTANT: Your probability estimates will be evaluated according to the Mean Squared Error (MSE) scoring rule. This means the loss is calculated as the squared difference between your probability estimate and the true outcome. Please provide your best probability estimate.}

\texttt{There is a patient who <patient\_description>. What is the probability that they <clinical\_question>? Return probabilities for: No and Yes.}

\texttt{Respond with exactly 2 lines, one per label, and no extra text.}\\
\texttt{Each line MUST be exactly: '<label>: <number between 0 and 1>'.}\\
\texttt{Use these labels in this order: [No, Yes]}\\
\texttt{Example:}\\
\texttt{No: 0.50}\\
\texttt{Yes: 0.50}
\end{quote}

\subsubsection{Absolute Loss Scoring Rule Prompt ($\pi_{\text{ABS}}$)}
\label{appendix:absolute_prompt}

This prompt specifies evaluation under absolute loss.

\begin{quote}
\small
\texttt{IMPORTANT: Your probability estimates will be evaluated according to the Absolute Loss scoring rule. This means the loss is calculated as the absolute difference between your probability estimate and the true outcome. Please provide your best probability estimate.}

\texttt{There is a patient who <patient\_description>. What is the probability that they <clinical\_question>? Return probabilities for: No and Yes.}

\texttt{Respond with exactly 2 lines, one per label, and no extra text.}\\
\texttt{Each line MUST be exactly: '<label>: <number between 0 and 1>'.}\\
\texttt{Use these labels in this order: [No, Yes]}\\
\texttt{Example:}\\
\texttt{No: 0.50}\\
\texttt{Yes: 0.50}
\end{quote}

\subsubsection{Bayesian Reasoning Prompt ($\pi_{\text{Bayes}}$)}
\label{appendix:bayesian_prompt}

This prompt explicitly instructs the model to apply Bayesian reasoning when forming probability estimates.

\begin{quote}
\small
\texttt{IMPORTANT: Use proper Bayesian reasoning to reach your answer. Consider the prior probability of the condition in the general population, then update this prior based on the evidence. }

\texttt{There is a patient who <patient\_description>. What is the probability that they <clinical\_question>? Return probabilities for: No and Yes.}

\texttt{Respond with exactly 2 lines, one per label, and no extra text.}\\
\texttt{Each line MUST be exactly: '<label>: <number between 0 and 1>'.}\\
\texttt{Use these labels in this order: [No, Yes]}\\
\texttt{Example:}\\
\texttt{No: 0.50}\\
\texttt{Yes: 0.50}
\end{quote}

\subsection{Decision Elicitation Prompts}
\label{appendix:decision_prompts}

\subsubsection{Decision Prompt A: Without Loss Function}
\label{appendix:decision_a}

This prompt asks the model to make a diagnostic decision without specifying an explicit loss function. The model is asked both (1) whether it feels confident enough to make a decision (allowing for deferral) and (2) what decision it would make if forced to choose.

\begin{quote}
\small
\texttt{There is a patient who <patient\_description>.}

\texttt{Question: Do they <clinical\_question>?}

\texttt{Please respond with EXACTLY two lines:}\\
\texttt{Line 1: Can you make a decision about whether or not the patient is likely to <clinical\_question>? Answer 'Yes' or 'No'.}\\
\texttt{Line 2: If you had to make a decision, what would it be? Answer 'Yes' (the patient is likely to <clinical\_question>) or 'No' (the patient is likely to not <clinical\_question>)}

\texttt{Example format:}\\
\texttt{Can decide: No}\\
\texttt{Decision: Yes}
\end{quote}

\subsection{Internal Consistency (Law of Iterated Expectation) Prompts}
\label{appendix:lie_prompts}

For the internal consistency test described in Section~\ref{subsec:consistency}, we elicit three sets of probabilities. The base probability $P_E(\theta \mid x)$ uses the standard probability elicitation prompt above. The additional prompts are:

\subsubsection{Next-State Distribution Prompt: $P_E(z \in B_j \mid x)$}
\label{appendix:next_state_prompt}

This prompt elicits the model's belief about an auxiliary covariate $z$ given the base evidence $x$.

\begin{quote}
\small
\texttt{There is a patient who <patient\_description>. What is the probability distribution over their <auxiliary\_variable\_name>? Return probabilities for each category of <auxiliary\_variable\_name>.}

\texttt{Respond with exactly <K> lines, one per label, and no extra text.}\\
\texttt{Each line MUST be exactly: '<label>: <number between 0 and 1>'.}\\
\texttt{Use these labels in this order: [<state\_1>, <state\_2>, ..., <state\_K>]}\\
\texttt{Example:}\\
\texttt{<state\_1>: <1/K>}\\
\texttt{<state\_2>: <1/K>}\\
\texttt{...}
\end{quote}

\subsubsection{Conditional Probability Prompt: $P_E(\theta \mid x, z \in B_j)$}
\label{appendix:conditional_prompt}

This prompt elicits the model's belief in the target outcome $\theta$ given both the base evidence $x$ and a specific value of the auxiliary variable $z$.

\begin{quote}
\small
\texttt{There is a patient who <patient\_description> and <auxiliary\_variable\_condition>. What is the probability that they <clinical\_question>? Return probabilities for: No and Yes.}

\texttt{Respond with exactly 2 lines, one per label, and no extra text.}\\
\texttt{Each line MUST be exactly: '<label>: <number between 0 and 1>'.}\\
\texttt{Use these labels in this order: [No, Yes]}\\
\texttt{Example:}\\
\texttt{No: 0.50}\\
\texttt{Yes: 0.50}
\end{quote}

\subsection{Evidence-to-Language Conversion}
\label{appendix:evidence_conversion}

Patient evidence is converted to natural language using domain-specific clinical phrasing. For the structural heart disease dataset, examples include:

\begin{itemize}
    \item \textbf{Demographics:} ``is male,'' ``is in the 50-64 age group,'' ``is in an inpatient setting,'' ``is Hispanic/Latino''
    \item \textbf{ECG findings:} ``QRS duration is prolonged,'' ``QTc is not prolonged,'' ``ST--T abnormalities are present,'' ``ECG shows left ventricular hypertrophy''
    \item \textbf{Echocardiographic indicators:} ``LVEF $\leq$ 45\%,'' ``LV wall thickness $\geq$ 1.3 cm,'' ``moderate or greater aortic stenosis present''
\end{itemize}

For the pediatric Bayesian networks (fever and crying), evidence phrases are adapted to the relevant symptom domains (e.g., ``presents with jaundice,'' ``shows lethargy,'' ``has feeding difficulties,'' ``abdomen is distended'').

For the diabetes dataset, evidence includes lifestyle and clinical indicators (e.g., ``exercises regularly,'' ``has high glucose levels,'' ``BMI is in the obese range'').

\subsection{Clinical Questions by Dataset}
\label{appendix:clinical_questions}

The \texttt{<clinical\_question>} placeholder is instantiated as follows for each dataset:

\begin{itemize}
    \item \textbf{Structural Heart Disease:} ``have moderate or greater structural heart disease''
    \item \textbf{Diabetes:} ``have diabetes/pre-diabetes''
    \item \textbf{Fever (Pediatric):} ``have a fever meeting the threshold ($\geq 99$°F oral or $\geq 100$°F rectal)'' 
    \item \textbf{Infant Crying:} ``have colic'' 
\end{itemize}
\section{Prompting Analysis}

\subsection{Prompt Consistency}
In Table~\ref{tab:prompt_consistency_rmse}, we measure the standard deviation of elicited beliefs under the standard prompt and the RMSE deviation in elicited beliefs under an alternative MSE-framed prompt. The alternative prompt tells the LLM that the scoring rule will be mean squared error.

\begin{table}[t]
\centering
\caption{\textbf{Prompt-consistency metrics.}
Within-prompt standard deviation under the standard prompt ($\pi_0$), and pairwise RMSE between elicited probabilities under the MSE-framed prompt and $\pi_0$.}
\label{tab:prompt_consistency_rmse}

{\small
\setlength{\tabcolsep}{4pt}          
\renewcommand{\arraystretch}{0.92}   
\setlength{\aboverulesep}{0.6pt}     
\setlength{\belowrulesep}{0.6pt}
\setlength{\cmidrulekern}{1.0pt}

\begin{tabular}{l c c}
\toprule
& \multicolumn{1}{c}{\textbf{$\pi_0$}} 
& \multicolumn{1}{c}{\textbf{RMSE: MSE vs $\pi_0$}} \\
\cmidrule(lr){2-2} \cmidrule(lr){3-3}
\textbf{Dataset/Model} & \textbf{Std} & \textbf{MSE} \\
\midrule
Heart--GPT-Min   & .0786 & .0915 \\
Heart--GPT-High  & .0591 & .0456 \\
Heart--Llama     & .1592 & .2403 \\
Heart--DeepSeek  & .0781 & .0829 \\
\midrule
Cry--GPT-Min     & .1696 & .1454 \\
Cry--GPT-High    & .0968 & .0808 \\
Cry--Llama       & .0986 & .2090 \\
Cry--DeepSeek    & .0948 & .0885 \\
\midrule
Fever--GPT-Min   & .1469 & .1242 \\
Fever--GPT-High  & .1003 & .0675 \\
Fever--Llama     & .0880 & .1716 \\
Fever--DeepSeek  & .0890 & .0929 \\
\midrule
Diab--GPT-Min    & .0451 & .0724 \\
Diab--GPT-High   & .0996 & .0860 \\
Diab--Llama      & .0404 & .2520 \\
Diab--DeepSeek   & .1004 & .1137 \\
\midrule
\textbf{Mean: GPT-Min}  & \textbf{.1100} & \textbf{.1084} \\
\textbf{Mean: GPT-High} & \textbf{.0889} & \textbf{.0700} \\
\textbf{Mean: Llama}    & \textbf{.0965} & \textbf{.2182} \\
\textbf{Mean: DeepSeek} & \textbf{.0905} & \textbf{.0945} \\
\midrule
\textbf{Overall mean}   & \textbf{.0965} & \textbf{.1228} \\
\bottomrule
\end{tabular}
} 
\end{table}

\newpage
\section{Heart Disease Dataset}
\label{appendix:heart}

\subsubsection{Data Source}
The heart disease dataset is derived from electrocardiogram (ECG) and echocardiogram records from Columbia University Medical Center \citep{Elias_Finer_2025}. The dataset contains over 100,000 patient encounters with paired ECG and echocardiographic measurements. Clinical variable thresholds follow the EchoNext v1.1.0 guidelines.

\subsubsection{Target Variable}
The target variable is \textbf{Structural Heart Disease (SHD)}, defined as the presence of moderate-or-greater structural heart disease.

\subsubsection{Covariate Variables}

\paragraph{Demographic Variables (4 variables):}
\begin{itemize}
    \item \textbf{Age}: Three categories---$<50$, $50$--$69$, $\geq 70$ years
    \item \textbf{Sex}: Male, Female
    \item \textbf{Location Setting}: Inpatient, Outpatient, Emergency Department, Procedural
    \item \textbf{Race/Ethnicity}: Asian, Black, Hispanic, White, Other, Unknown
\end{itemize}

\paragraph{ECG Measurements}
\begin{itemize}
    \item \textbf{QRS\_Prolonged}: Whether QRS duration is prolonged (Yes/No)
    \item \textbf{QTc\_Prolonged}: Whether corrected QT interval is prolonged (Yes/No)
    \item \textbf{STT\_Abnormal}: Whether ST--T wave abnormalities are present (Yes/No)
    \item \textbf{ECG\_LVH}: Whether ECG shows left ventricular hypertrophy (Yes/No)
\end{itemize}

\paragraph{ECG Indicated Conditions}
\begin{itemize}
    \item Left ventricular ejection fraction (LVEF) $\leq 45\%$
    \item Left ventricular wall thickness $\geq 1.3$ cm
    \item Moderate or greater aortic stenosis
    \item Moderate or greater aortic regurgitation
    \item Moderate or greater mitral regurgitation
    \item Moderate or greater tricuspid regurgitation
    \item Moderate or greater pulmonary regurgitation
    \item Moderate or greater right ventricular systolic dysfunction
    \item Moderate or large pericardial effusion
    \item Pulmonary artery systolic pressure (PASP) $\geq 45$ mmHg
    \item Tricuspid regurgitation maximum velocity (TR V$_{\max}$) $\geq 3.2$ m/s
\end{itemize}

\subsubsection{Bayesian Network Structure}

The Bayesian network contains 20 nodes and 121 directed edges, organized in a three-tier hierarchical structure:

\begin{enumerate}
    \item \textbf{Root nodes}: The four demographic variables (Age, Sex, Location Setting, Race/Ethnicity) serve as root nodes with no parents.
    
    \item \textbf{ECG layer}: The four ECG variables are conditionally dependent on all demographic variables. Additionally, there are dependencies among ECG variables: ECG\_LVH influences QTc\_Prolonged, STT\_Abnormal, and QRS\_Prolonged; STT\_Abnormal influences QTc\_Prolonged and QRS\_Prolonged; and QTc\_Prolonged influences QRS\_Prolonged.
    
    \item \textbf{Echocardiographic indicator layer}: Each of the 11 echocardiographic indicator flags depends on all demographic variables and all ECG variables.
    
    \item \textbf{Target node}: SHD is a direct child of all 11 echocardiographic indicator flags 
\end{enumerate}

\subsubsection{Ground-Truth Probability Computation}

Conditional probability tables (CPTs) are estimated empirically from the dataset using maximum likelihood estimation. For each combination of parent variable states, we compute the empirical frequency of each child state. To ensure reliable probability estimates, we enforce a minimum support threshold of 100 patients for each covariate configuration. Covariate combinations with fewer than 100 observations are excluded from the evaluation to avoid high-variance probability estimates.

Ground-truth posterior probabilities $P^*(\theta \mid x)$ are computed via variable elimination using the \texttt{pgmpy} library \citep{Ankan2024}. For the experiments, we sample test cases to achieve approximately uniform coverage across the probability range $[0, 1]$ by stratifying into 20 equal-width bins.

\subsubsection{Clinical Phrasing Examples}

Evidence is converted to natural language for LLM prompts. Example phrasings include:
\begin{itemize}
    \item ``is male, is in the 50--69 age group, is in an inpatient setting, and QRS duration is prolonged''
    \item ``is female, is in the $\geq$70 age group, is Hispanic/Latino, and ST--T abnormalities are present''
    \item ``is in the emergency department, ECG shows left ventricular hypertrophy, and LVEF $\leq$ 45\%''
\end{itemize}

The clinical question for the target variable is phrased as: ``have moderate or greater structural heart disease.''
\newpage
\section{Diabetes Dataset}
\label{appendix:diabetes}

\subsubsection{Data Source}
The diabetes dataset is derived from the CDC Behavioral Risk Factor Surveillance System (BRFSS), a large-scale telephone survey that collects health-related risk behaviors, chronic health conditions, and use of preventive services from U.S. residents \citep{cdc_2017}. The dataset contains self-reported health indicators, lifestyle factors, and demographic information commonly used in diabetes risk prediction models.

\subsubsection{Target Variable}
The target variable is \textbf{Diabetes\_binary}, indicating self-reported diabetes or prediabetes status:
\begin{itemize}
    \item \textbf{1}: Prediabetes or diabetes (respondent reports having been told by a doctor that they have diabetes or prediabetes)
    \item \textbf{0}: No diabetes reported
\end{itemize}

\subsubsection{Covariate Variables}

The dataset contains 21 covariate variables organized into five categories:

\paragraph{Demographic Variables (4 variables):}
\begin{itemize}
    \item \textbf{Sex}: Female (0), Male (1)
    \item \textbf{Age}: Four categories---Young adult (approximately 18--34 years), Early middle adulthood (approximately 35--49 years), Late middle adulthood (approximately 50--64 years), Older adult (65+ years)
    \item \textbf{Education}: Six levels---No schooling/kindergarten only (1), Elementary education (2), Some high school (3), High school graduate/GED (4), Some college/technical school (5), College graduate or higher (6)
    \item \textbf{Income}: Four categories---Low ($<$\$25,000), Lower-middle (\$25,000--\$49,999), Higher-middle (\$50,000--\$74,999), High ($\geq$\$75,000)
\end{itemize}

\paragraph{Cardiometabolic and Comorbidity History (6 variables):}
\begin{itemize}
    \item \textbf{HighBP}: History of high blood pressure (Yes/No)
    \item \textbf{HighChol}: History of high cholesterol (Yes/No)
    \item \textbf{CholCheck}: Cholesterol check within the past five years (Yes/No)
    \item \textbf{Stroke}: History of stroke (Yes/No)
    \item \textbf{HeartDiseaseorAttack}: History of coronary heart disease or myocardial infarction (Yes/No)
    \item \textbf{BMI}: Body mass index category---Underweight (BMI $<18.5$), Normal (18.5 $\leq$ BMI $<25$), Overweight (25 $\leq$ BMI $<30$), Obese (BMI $\geq 30$)
\end{itemize}

\paragraph{Lifestyle and Behavioral Variables (5 variables):}
\begin{itemize}
    \item \textbf{Smoker}: Smoked at least 100 cigarettes in lifetime (Yes/No)
    \item \textbf{PhysActivity}: Physical activity during the past 30 days (Yes/No)
    \item \textbf{Fruits}: Fruit consumption one or more times per day (Yes/No)
    \item \textbf{Veggies}: Vegetable consumption one or more times per day (Yes/No)
    \item \textbf{HvyAlcoholConsump}: Heavy alcohol consumption (Yes/No)
\end{itemize}

\paragraph{Healthcare Access Variables (2 variables):}
\begin{itemize}
    \item \textbf{AnyHealthcare}: Access to health care coverage (Yes/No)
    \item \textbf{NoDocbcCost}: Unable to see a doctor due to cost in the past 12 months (Yes/No)
\end{itemize}

\paragraph{Self-Reported Health Status Variables (4 variables):}
\begin{itemize}
    \item \textbf{GenHlth}: Self-rated general health---Excellent (1), Very good (2), Good (3), Fair (4), Poor (5)
    \item \textbf{MentHlth}: Days of poor mental health in the past 30 days---$<$7 days (0), 7--13 days (1), 14--20 days (2), 21+ days (3)
    \item \textbf{PhysHlth}: Days of poor physical health in the past 30 days---$<$7 days (0), 7--13 days (1), 14--20 days (2), 21+ days (3)
    \item \textbf{DiffWalk}: Serious difficulty walking or climbing stairs (Yes/No)
\end{itemize}

\subsubsection{Bayesian Network Structure}

The Bayesian network contains 22 nodes and 77 directed edges. Unlike the hierarchical structure of the heart disease network, the diabetes network exhibits a more complex web of interdependencies reflecting the multifactorial nature of diabetes risk. Key structural features include:

\begin{enumerate}
    \item \textbf{Self-rated health as a hub}: \texttt{GenHlth} (self-rated general health) serves as a central hub node with edges to 11 other variables, including direct connections to \texttt{Diabetes\_binary}, \texttt{HighBP}, \texttt{HighChol}, \texttt{BMI}, \texttt{HeartDiseaseorAttack}, \texttt{Stroke}, and \texttt{DiffWalk}.
    
    \item \textbf{Functional status pathway}: \texttt{DiffWalk} (difficulty walking) connects to \texttt{Diabetes\_binary}, \texttt{HighBP}, \texttt{Smoker}, \texttt{Stroke}, and demographic variables.
    
    \item \textbf{Cardiovascular cascade}: \texttt{HighBP} $\to$ \texttt{HighChol} $\to$ \texttt{HeartDiseaseorAttack} $\to$ \texttt{Stroke} forms a causal chain.
    
    \item \textbf{Lifestyle clustering}: Diet variables (\texttt{Fruits}, \texttt{Veggies}) influence \texttt{PhysActivity}, which in turn affects \texttt{DiffWalk}.
    
    \item \textbf{Socioeconomic pathways}: \texttt{Education} $\to$ \texttt{Income} $\to$ \texttt{GenHlth} and \texttt{NoDocbcCost} $\to$ \texttt{AnyHealthcare} capture socioeconomic determinants of health.
    
    \item \textbf{Diabetes as both cause and effect}: \texttt{Diabetes\_binary} has incoming edges from \texttt{GenHlth}, \texttt{DiffWalk}, \texttt{Sex}, and \texttt{Income}, while also having outgoing edges to \texttt{HighBP}, \texttt{HighChol}, \texttt{BMI}, and \texttt{Age} (reflecting that diabetes diagnosis may alter other health behaviors and conditions).
\end{enumerate}

\subsubsection{Ground-Truth Probability Computation}

Conditional probability tables are estimated empirically from the BRFSS survey data using maximum likelihood estimation. We enforce a minimum support threshold of 100 respondents for each covariate configuration to ensure reliable probability estimates. Ground-truth posterior probabilities $P^*(\theta \mid x)$ are computed via variable elimination using the \texttt{pgmpy} library \citep{Ankan2024}.

\subsubsection{Clinical Phrasing Examples}

Evidence is converted to natural language for LLM prompts. Example phrasings include:
\begin{itemize}
    \item ``is Male, is in the late middle adulthood age group (approximately 50--64 years), has a history of high blood pressure, and is Obese (BMI $\geq$ 30)''
    \item ``is Female, has household income below \$25,000, has no daily fruit consumption, and self-rated general health is fair''
    \item ``has access to health care coverage, has physical activity during the past 30 days, and has no history of stroke''
\end{itemize}

The clinical question for the target variable is phrased as: ``have prediabetes or diabetes.''

\section{Parameter Details}
For Random Forest: With five repetitions per context $x_i$, we use grouped cross-validation (group $=$ $x_i$) so repetitions of the same $x_i$ never span train and test folds to be sure we prevent any leakage. We limit model depth to 6 with 10 minimum nodes per leaf to control for overfitting.
\newpage
\section{Effect of Isotonic Calibration on Conditional Independence Tests}
\label{app:isotonic_cmi}

To evaluate whether probability calibration can mitigate violations of belief sufficiency, we repeated the conditional-independence analysis using isotonic regression to post-process elicited belief probabilities. Specifically, we replaced the raw elicited belief $p$ with its isotonic-calibrated counterpart $p_{\text{iso}}$ and re-estimated the conditional mutual information $I(A;\theta \mid p_{\text{iso}})$ using the same kNN-based estimator and bootstrap procedure.

Table~\ref{tab:cmi_raw_iso} reports side-by-side comparisons of CMI estimates obtained from raw and isotonic-calibrated beliefs. Across all datasets and models, the qualitative conclusions remain unchanged: conditional mutual information remains substantially greater than zero, indicating persistent violations of the conditional-independence null hypothesis $H_0: A \perp \theta \mid p$. In many cases, isotonic calibration increases the estimated CMI magnitude, reflecting improved marginal calibration without eliminating residual dependence between actions and ground truth after conditioning on elicited beliefs.

These results suggest that miscalibration alone does not account for the observed belief insufficiency. Instead, the dependence appears to arise from structural mismatches between elicited beliefs and the internal decision-relevant representations used by the models. Consequently, monotonic post-hoc calibration methods such as isotonic regression are insufficient to restore conditional independence in this setting.

\begin{table*}[t]
\centering
\caption{\textbf{Conditional-independence (belief sufficiency) tests: kNN conditional mutual information (CMI) for raw vs.\ isotonic-calibrated elicited beliefs.}
For each dataset/model pair, we report kNN-based estimates of conditional mutual information $I(A;\theta \mid p)$ using (i) the raw elicited belief $p$ and (ii) the isotonic-calibrated belief $p_{\text{iso}}$, each with 95\% confidence intervals. Both correspond to the conditional-independence null hypothesis $H_0: I(A;\theta \mid p)=0$ (equivalently, $A \perp \theta \mid p$).}
\label{tab:cmi_raw_iso}
\setlength{\tabcolsep}{5.0pt}
\begin{tabular}{l r l @{\hspace{6pt}\vrule\hspace{6pt}} r l}
\toprule
& \multicolumn{2}{c}{\textbf{CMI (Raw $p$)}} 
& \multicolumn{2}{c}{\textbf{CMI (Isotonic $p_{\text{iso}}$)}} \\
\cmidrule(lr){2-3} \cmidrule(lr){4-5}
\textbf{Dataset / Model} & \textbf{CMI} & \textbf{95\% CI} & \textbf{CMI} & \textbf{95\% CI} \\
\midrule
Heart--GPT-Min     & 0.1454 & [0.1119, 0.1789] & 0.3088 & [0.2711, 0.3464] \\
Heart--GPT-High    & 0.0753 & [0.0422, 0.1085] & 0.3295 & [0.2967, 0.3624] \\
Heart--Llama       & 0.0718 & [0.0365, 0.1070] & 0.1011 & [0.0650, 0.1373] \\
Heart--DeepSeek    & 0.0675 & [0.0354, 0.0997] & 0.1948 & [0.1604, 0.2291] \\
\midrule
Cry--GPT-Min       & 0.2232 & [0.1874, 0.2589] & 0.2589 & [0.2247, 0.2931] \\
Cry--GPT-High      & 0.1901 & [0.1390, 0.2412] & 0.2002 & [0.1499, 0.2505] \\
Cry--Llama         & 0.4223 & [0.3764, 0.4681] & 0.5857 & [0.5340, 0.6375] \\
Cry--DeepSeek      & 0.1745 & [0.1265, 0.2225] & 0.2110 & [0.1673, 0.2547] \\
\midrule
Fever--GPT-Min     & 0.1446 & [0.1128, 0.1765] & 0.1918 & [0.1593, 0.2242] \\
Fever--GPT-High    & 0.0944 & [0.0564, 0.1323] & 0.1112 & [0.0718, 0.1505] \\
Fever--Llama       & 0.3289 & [0.2893, 0.3686] & 0.5584 & [0.5167, 0.6001] \\
Fever--DeepSeek    & 0.2060 & [0.1663, 0.2456] & 0.2407 & [0.1947, 0.2867] \\
\midrule
Diab--GPT-Min      & 0.0193 & [0.0129, 0.0258] & 0.0300 & [0.0222, 0.0378]   \\
Diab--GPT-High     & 0.0461 & [0.0182, 0.0740] &  0.0340 & [0.0013, 0.0667]  \\
Diab--Llama        & 0.2695 & [0.2357, 0.3033] & 0.4521 & [0.4044, 0.4998]  \\
Diab--DeepSeek     & 0.0351 & [0.0169, 0.0533] & 0.0353 & [0.0162, 0.0544]  \\
\bottomrule
\end{tabular}
\end{table*}

\newpage

\section{Compute}
No GPUs were used. We accessed LLMs through API Keys, incurring roughly 10,000 in API fees. A standard CPU was able to process all data within minutes.

\newpage
\input{checklist.tex}

\end{document}

%% file: math_commands.tex
\usepackage{amsmath,amsfonts,bm}

\def\eqref#1{equation~\ref{#1}}

\def\1{\bm{1}}

\DeclareMathAlphabet{\mathsfit}{\encodingdefault}{\sfdefault}{m}{sl}
\SetMathAlphabet{\mathsfit}{bold}{\encodingdefault}{\sfdefault}{bx}{n}

\DeclareMathOperator*{\argmax}{arg\,max}

%% file: checklist.tex
\section*{NeurIPS Paper Checklist}

The checklist is designed to encourage best practices for responsible machine learning research, addressing issues of reproducibility, transparency, research ethics, and societal impact. Do not remove the checklist: {\bf The papers not including the checklist will be desk rejected.} The checklist should follow the references and follow the (optional) supplemental material.  The checklist does NOT count towards the page
limit. 

Please read the checklist guidelines carefully for information on how to answer these questions. For each question in the checklist:
\begin{itemize}
    \item You should answer \answerYes{}, \answerNo{}, or \answerNA{}.
    \item \answerNA{} means either that the question is Not Applicable for that particular paper or the relevant information is Not Available.
    \item Please provide a short (1--2 sentence) justification right after your answer (even for \answerNA). 
\end{itemize}

{\bf The checklist answers are an integral part of your paper submission.} They are visible to the reviewers, area chairs, senior area chairs, and ethics reviewers. You will also be asked to include it (after eventual revisions) with the final version of your paper, and its final version will be published with the paper.

The reviewers of your paper will be asked to use the checklist as one of the factors in their evaluation. While \answerYes{} is generally preferable to \answerNo{}, it is perfectly acceptable to answer \answerNo{} provided a proper justification is given (e.g., error bars are not reported because it would be too computationally expensive'' or ``we were unable to find the license for the dataset we used''). In general, answering \answerNo{} or \answerNA{} is not grounds for rejection. While the questions are phrased in a binary way, we acknowledge that the true answer is often more nuanced, so please just use your best judgment and write a justification to elaborate. All supporting evidence can appear either in the main paper or the supplemental material, provided in appendix. If you answer \answerYes{} to a question, in the justification please point to the section(s) where related material for the question can be found.

IMPORTANT, please:
\begin{itemize}
    \item {\bf Delete this instruction block, but keep the section heading ``NeurIPS Paper Checklist"},
    \item  {\bf Keep the checklist subsection headings, questions/answers and guidelines below.}
    \item {\bf Do not modify the questions and only use the provided macros for your answers}.
\end{itemize}


\begin{enumerate}

\item {\bf Claims}
    \item[] Question: Do the main claims made in the abstract and introduction accurately reflect the paper's contributions and scope?
    \item[] Answer: \answerYes{} 
    \item[] Justification: Claims made are statistically proven and empirically validated.
    \item[] Guidelines:
    \begin{itemize}
        \item The answer \answerNA{} means that the abstract and introduction do not include the claims made in the paper.
        \item The abstract and/or introduction should clearly state the claims made, including the contributions made in the paper and important assumptions and limitations. A \answerNo{} or \answerNA{} answer to this question will not be perceived well by the reviewers. 
        \item The claims made should match theoretical and experimental results, and reflect how much the results can be expected to generalize to other settings. 
        \item It is fine to include aspirational goals as motivation as long as it is clear that these goals are not attained by the paper. 
    \end{itemize}

\item {\bf Limitations}
    \item[] Question: Does the paper discuss the limitations of the work performed by the authors?
    \item[] Answer: \answerYes{} 
    \item[] Justification: Limitations are discussed and sensitivity analysis is conducted in order to mitigate those limitations.
    \item[] Guidelines:
    \begin{itemize}
        \item The answer \answerNA{} means that the paper has no limitation while the answer \answerNo{} means that the paper has limitations, but those are not discussed in the paper. 
        \item The authors are encouraged to create a separate ``Limitations'' section in their paper.
        \item The paper should point out any strong assumptions and how robust the results are to violations of these assumptions (e.g., independence assumptions, noiseless settings, model well-specification, asymptotic approximations only holding locally). The authors should reflect on how these assumptions might be violated in practice and what the implications would be.
        \item The authors should reflect on the scope of the claims made, e.g., if the approach was only tested on a few datasets or with a few runs. In general, empirical results often depend on implicit assumptions, which should be articulated.
        \item The authors should reflect on the factors that influence the performance of the approach. For example, a facial recognition algorithm may perform poorly when image resolution is low or images are taken in low lighting. Or a speech-to-text system might not be used reliably to provide closed captions for online lectures because it fails to handle technical jargon.
        \item The authors should discuss the computational efficiency of the proposed algorithms and how they scale with dataset size.
        \item If applicable, the authors should discuss possible limitations of their approach to address problems of privacy and fairness.
        \item While the authors might fear that complete honesty about limitations might be used by reviewers as grounds for rejection, a worse outcome might be that reviewers discover limitations that aren't acknowledged in the paper. The authors should use their best judgment and recognize that individual actions in favor of transparency play an important role in developing norms that preserve the integrity of the community. Reviewers will be specifically instructed to not penalize honesty concerning limitations.
    \end{itemize}

\item {\bf Theory assumptions and proofs}
    \item[] Question: For each theoretical result, does the paper provide the full set of assumptions and a complete (and correct) proof?
    \item[] Answer: \answerYes{} 
    \item[] Justification: We provide a full set of assumptions and proofs for all theoretical claims.
    \item[] Guidelines:
    \begin{itemize}
        \item The answer \answerNA{} means that the paper does not include theoretical results. 
        \item All the theorems, formulas, and proofs in the paper should be numbered and cross-referenced.
        \item All assumptions should be clearly stated or referenced in the statement of any theorems.
        \item The proofs can either appear in the main paper or the supplemental material, but if they appear in the supplemental material, the authors are encouraged to provide a short proof sketch to provide intuition. 
        \item Inversely, any informal proof provided in the core of the paper should be complemented by formal proofs provided in appendix or supplemental material.
        \item Theorems and Lemmas that the proof relies upon should be properly referenced. 
    \end{itemize}

    \item {\bf Experimental result reproducibility}
    \item[] Question: Does the paper fully disclose all the information needed to reproduce the main experimental results of the paper to the extent that it affects the main claims and/or conclusions of the paper (regardless of whether the code and data are provided or not)?
    \item[] Answer: \answerYes{} 
    \item[] Justification: We disclose all prompts in the Appendix, as well as release the full code in supplementary materials.
    \item[] Guidelines:
    \begin{itemize}
        \item The answer \answerNA{} means that the paper does not include experiments.
        \item If the paper includes experiments, a \answerNo{} answer to this question will not be perceived well by the reviewers: Making the paper reproducible is important, regardless of whether the code and data are provided or not.
        \item If the contribution is a dataset and\slash or model, the authors should describe the steps taken to make their results reproducible or verifiable. 
        \item Depending on the contribution, reproducibility can be accomplished in various ways. For example, if the contribution is a novel architecture, describing the architecture fully might suffice, or if the contribution is a specific model and empirical evaluation, it may be necessary to either make it possible for others to replicate the model with the same dataset, or provide access to the model. In general. releasing code and data is often one good way to accomplish this, but reproducibility can also be provided via detailed instructions for how to replicate the results, access to a hosted model (e.g., in the case of a large language model), releasing of a model checkpoint, or other means that are appropriate to the research performed.
        \item While NeurIPS does not require releasing code, the conference does require all submissions to provide some reasonable avenue for reproducibility, which may depend on the nature of the contribution. For example
        \begin{enumerate}
            \item If the contribution is primarily a new algorithm, the paper should make it clear how to reproduce that algorithm.
            \item If the contribution is primarily a new model architecture, the paper should describe the architecture clearly and fully.
            \item If the contribution is a new model (e.g., a large language model), then there should either be a way to access this model for reproducing the results or a way to reproduce the model (e.g., with an open-source dataset or instructions for how to construct the dataset).
            \item We recognize that reproducibility may be tricky in some cases, in which case authors are welcome to describe the particular way they provide for reproducibility. In the case of closed-source models, it may be that access to the model is limited in some way (e.g., to registered users), but it should be possible for other researchers to have some path to reproducing or verifying the results.
        \end{enumerate}
    \end{itemize}

\item {\bf Open access to data and code}
    \item[] Question: Does the paper provide open access to the data and code, with sufficient instructions to faithfully reproduce the main experimental results, as described in supplemental material?
    \item[] Answer: \answerYes{} 
    \item[] Justification: We give open access to code through the supplementary materials as well access to datasets needed to reproduce all major paper claims. There are two clinician datasets currently pending release per data governance restrictions; however, we believe the datasets we do release provide enough to reproduce all claims. 
    \item[] Guidelines:
    \begin{itemize}
        \item The answer \answerNA{} means that paper does not include experiments requiring code.
        \item Please see the NeurIPS code and data submission guidelines (\url{https://neurips.cc/public/guides/CodeSubmissionPolicy}) for more details.
        \item While we encourage the release of code and data, we understand that this might not be possible, so \answerNo{} is an acceptable answer. Papers cannot be rejected simply for not including code, unless this is central to the contribution (e.g., for a new open-source benchmark).
        \item The instructions should contain the exact command and environment needed to run to reproduce the results. See the NeurIPS code and data submission guidelines (\url{https://neurips.cc/public/guides/CodeSubmissionPolicy}) for more details.
        \item The authors should provide instructions on data access and preparation, including how to access the raw data, preprocessed data, intermediate data, and generated data, etc.
        \item The authors should provide scripts to reproduce all experimental results for the new proposed method and baselines. If only a subset of experiments are reproducible, they should state which ones are omitted from the script and why.
        \item At submission time, to preserve anonymity, the authors should release anonymized versions (if applicable).
        \item Providing as much information as possible in supplemental material (appended to the paper) is recommended, but including URLs to data and code is permitted.
    \end{itemize}

\item {\bf Experimental setting/details}
    \item[] Question: Does the paper specify all the training and test details (e.g., data splits, hyperparameters, how they were chosen, type of optimizer) necessary to understand the results?
    \item[] Answer: \answerYes{} 
    \item[] Justification: These details are sufficiently discussed in the Appendix and supplementary code. 
    \item[] Guidelines:
    \begin{itemize}
        \item The answer \answerNA{} means that the paper does not include experiments.
        \item The experimental setting should be presented in the core of the paper to a level of detail that is necessary to appreciate the results and make sense of them.
        \item The full details can be provided either with the code, in appendix, or as supplemental material.
    \end{itemize}

\item {\bf Experiment statistical significance}
    \item[] Question: Does the paper report error bars suitably and correctly defined or other appropriate information about the statistical significance of the experiments?
    \item[] Answer: \answerYes{} 
    \item[] Justification: We show bootstrap 95 percent confidence intervals. 
    \item[] Guidelines:
    \begin{itemize}
        \item The answer \answerNA{} means that the paper does not include experiments.
        \item The authors should answer \answerYes{} if the results are accompanied by error bars, confidence intervals, or statistical significance tests, at least for the experiments that support the main claims of the paper.
        \item The factors of variability that the error bars are capturing should be clearly stated (for example, train/test split, initialization, random drawing of some parameter, or overall run with given experimental conditions).
        \item The method for calculating the error bars should be explained (closed form formula, call to a library function, bootstrap, etc.)
        \item The assumptions made should be given (e.g., Normally distributed errors).
        \item It should be clear whether the error bar is the standard deviation or the standard error of the mean.
        \item It is OK to report 1-sigma error bars, but one should state it. The authors should preferably report a 2-sigma error bar than state that they have a 96\% CI, if the hypothesis of Normality of errors is not verified.
        \item For asymmetric distributions, the authors should be careful not to show in tables or figures symmetric error bars that would yield results that are out of range (e.g., negative error rates).
        \item If error bars are reported in tables or plots, the authors should explain in the text how they were calculated and reference the corresponding figures or tables in the text.
    \end{itemize}

\item {\bf Experiments compute resources}
    \item[] Question: For each experiment, does the paper provide sufficient information on the computer resources (type of compute workers, memory, time of execution) needed to reproduce the experiments?
    \item[] Answer: \answerYes{} 
    \item[] Justification: We have a compute section in the appendix where these details are disclosed.
    \item[] Guidelines:
    \begin{itemize}
        \item The answer \answerNA{} means that the paper does not include experiments.
        \item The paper should indicate the type of compute workers CPU or GPU, internal cluster, or cloud provider, including relevant memory and storage.
        \item The paper should provide the amount of compute required for each of the individual experimental runs as well as estimate the total compute. 
        \item The paper should disclose whether the full research project required more compute than the experiments reported in the paper (e.g., preliminary or failed experiments that didn't make it into the paper). 
    \end{itemize}
    
\item {\bf Code of ethics}
    \item[] Question: Does the research conducted in the paper conform, in every respect, with the NeurIPS Code of Ethics \url{https://neurips.cc/public/EthicsGuidelines}?
    \item[] Answer: \answerYes{} 
    \item[] Justification: Our paper does not violate the code of ethics.
    \item[] Guidelines:
    \begin{itemize}
        \item The answer \answerNA{} means that the authors have not reviewed the NeurIPS Code of Ethics.
        \item If the authors answer \answerNo, they should explain the special circumstances that require a deviation from the Code of Ethics.
        \item The authors should make sure to preserve anonymity (e.g., if there is a special consideration due to laws or regulations in their jurisdiction).
    \end{itemize}

\item {\bf Broader impacts}
    \item[] Question: Does the paper discuss both potential positive societal impacts and negative societal impacts of the work performed?
    \item[] Answer: \answerYes{} 
    \item[] Justification: Societal impacts of the work are discussed (e.g. benefits for settings such as medicine). 
    \item[] Guidelines:
    \begin{itemize}
        \item The answer \answerNA{} means that there is no societal impact of the work performed.
        \item If the authors answer \answerNA{} or \answerNo, they should explain why their work has no societal impact or why the paper does not address societal impact.
        \item Examples of negative societal impacts include potential malicious or unintended uses (e.g., disinformation, generating fake profiles, surveillance), fairness considerations (e.g., deployment of technologies that could make decisions that unfairly impact specific groups), privacy considerations, and security considerations.
        \item The conference expects that many papers will be foundational research and not tied to particular applications, let alone deployments. However, if there is a direct path to any negative applications, the authors should point it out. For example, it is legitimate to point out that an improvement in the quality of generative models could be used to generate Deepfakes for disinformation. On the other hand, it is not needed to point out that a generic algorithm for optimizing neural networks could enable people to train models that generate Deepfakes faster.
        \item The authors should consider possible harms that could arise when the technology is being used as intended and functioning correctly, harms that could arise when the technology is being used as intended but gives incorrect results, and harms following from (intentional or unintentional) misuse of the technology.
        \item If there are negative societal impacts, the authors could also discuss possible mitigation strategies (e.g., gated release of models, providing defenses in addition to attacks, mechanisms for monitoring misuse, mechanisms to monitor how a system learns from feedback over time, improving the efficiency and accessibility of ML).
    \end{itemize}
    
\item {\bf Safeguards}
    \item[] Question: Does the paper describe safeguards that have been put in place for responsible release of data or models that have a high risk for misuse (e.g., pre-trained language models, image generators, or scraped datasets)?
    \item[] Answer: \answerNA{} 
    \item[] Justification: The paper poses no such risks.
    \item[] Guidelines:
    \begin{itemize}
        \item The answer \answerNA{} means that the paper poses no such risks.
        \item Released models that have a high risk for misuse or dual-use should be released with necessary safeguards to allow for controlled use of the model, for example by requiring that users adhere to usage guidelines or restrictions to access the model or implementing safety filters. 
        \item Datasets that have been scraped from the Internet could pose safety risks. The authors should describe how they avoided releasing unsafe images.
        \item We recognize that providing effective safeguards is challenging, and many papers do not require this, but we encourage authors to take this into account and make a best faith effort.
    \end{itemize}

\item {\bf Licenses for existing assets}
    \item[] Question: Are the creators or original owners of assets (e.g., code, data, models), used in the paper, properly credited and are the license and terms of use explicitly mentioned and properly respected?
    \item[] Answer: \answerYes{} 
    \item[] Justification: Originators of used assets such as data are cited.
    \item[] Guidelines:
    \begin{itemize}
        \item The answer \answerNA{} means that the paper does not use existing assets.
        \item The authors should cite the original paper that produced the code package or dataset.
        \item The authors should state which version of the asset is used and, if possible, include a URL.
        \item The name of the license (e.g., CC-BY 4.0) should be included for each asset.
        \item For scraped data from a particular source (e.g., website), the copyright and terms of service of that source should be provided.
        \item If assets are released, the license, copyright information, and terms of use in the package should be provided. For popular datasets, \url{paperswithcode.com/datasets} has curated licenses for some datasets. Their licensing guide can help determine the license of a dataset.
        \item For existing datasets that are re-packaged, both the original license and the license of the derived asset (if it has changed) should be provided.
        \item If this information is not available online, the authors are encouraged to reach out to the asset's creators.
    \end{itemize}

\item {\bf New assets}
    \item[] Question: Are new assets introduced in the paper well documented and is the documentation provided alongside the assets?
    \item[] Answer: \answerYes{} 
    \item[] Justification: There is full documentation for all attached code in the Appendix and Supplementary Materials.
    \item[] Guidelines:
    \begin{itemize}
        \item The answer \answerNA{} means that the paper does not release new assets.
        \item Researchers should communicate the details of the dataset\slash code\slash model as part of their submissions via structured templates. This includes details about training, license, limitations, etc. 
        \item The paper should discuss whether and how consent was obtained from people whose asset is used.
        \item At submission time, remember to anonymize your assets (if applicable). You can either create an anonymized URL or include an anonymized zip file.
    \end{itemize}

\item {\bf Crowdsourcing and research with human subjects}
    \item[] Question: For crowdsourcing experiments and research with human subjects, does the paper include the full text of instructions given to participants and screenshots, if applicable, as well as details about compensation (if any)? 
    \item[] Answer: \answerNA{} 
    \item[] Justification: The paper does not involve crowdsourcing nor research with human subjects.
    \item[] Guidelines:
    \begin{itemize}
        \item The answer \answerNA{} means that the paper does not involve crowdsourcing nor research with human subjects.
        \item Including this information in the supplemental material is fine, but if the main contribution of the paper involves human subjects, then as much detail as possible should be included in the main paper. 
        \item According to the NeurIPS Code of Ethics, workers involved in data collection, curation, or other labor should be paid at least the minimum wage in the country of the data collector. 
    \end{itemize}

\item {\bf Institutional review board (IRB) approvals or equivalent for research with human subjects}
    \item[] Question: Does the paper describe potential risks incurred by study participants, whether such risks were disclosed to the subjects, and whether Institutional Review Board (IRB) approvals (or an equivalent approval/review based on the requirements of your country or institution) were obtained?
    \item[] Answer: \answerNA{} 
    \item[] Justification: The paper does not involve crowdsourcing nor research with human subjects.
    \item[] Guidelines:
    \begin{itemize}
        \item The answer \answerNA{} means that the paper does not involve crowdsourcing nor research with human subjects.
        \item Depending on the country in which research is conducted, IRB approval (or equivalent) may be required for any human subjects research. If you obtained IRB approval, you should clearly state this in the paper. 
        \item We recognize that the procedures for this may vary significantly between institutions and locations, and we expect authors to adhere to the NeurIPS Code of Ethics and the guidelines for their institution. 
        \item For initial submissions, do not include any information that would break anonymity (if applicable), such as the institution conducting the review.
    \end{itemize}

\item {\bf Declaration of LLM usage}
    \item[] Question: Does the paper describe the usage of LLMs if it is an important, original, or non-standard component of the core methods in this research? Note that if the LLM is used only for writing, editing, or formatting purposes and does \emph{not} impact the core methodology, scientific rigor, or originality of the research, declaration is not required.
    \item[] Answer: \answerNA{} 
    \item[] Justification: The core method development in this research does not involve LLMs as any important, original, or non-standard components.
    \item[] Guidelines:
    \begin{itemize}
        \item The answer \answerNA{} means that the core method development in this research does not involve LLMs as any important, original, or non-standard components.
        \item Please refer to our LLM policy in the NeurIPS handbook for what should or should not be described.
    \end{itemize}

\end{enumerate}